\documentclass[phd,tocprelim]{cornell}
\usepackage{graphicx,pstricks}
\usepackage{amsmath}
\usepackage{amssymb}
\usepackage{amsthm}
\usepackage{graphics}
\usepackage{moreverb}
\usepackage{subfigure}
\usepackage{epsfig}
\usepackage{subfigure}
\usepackage{txfonts}
\usepackage{mathpazo} % math & rm
\usepackage{hyperref}
\usepackage{booktabs}
\usepackage{xcolor,colortbl}
\usepackage[scaled]{helvet} % ss
\usepackage{courier} % tt
\normalfont
\usepackage[T1]{fontenc}
\usepackage{geometry}
\newcommand{\sentence}{\bar{s}}
\usepackage{url}
\usepackage{tikz-dependency}
\pgfkeys{%
/depgraph/reserved/edge style/.style = {% 
white, -, >=stealth, % arrow properties                                                                            
black, solid, line cap=round, % line properties
rounded corners=2, % make corners round
},%
}
\usepackage{tikz}
\usepackage{linguex}
\usepackage{tikz-qtree}
\usepackage{gillius}
\usepackage{floatrow}
\newcommand{\depmem}[1] {\textbf{#1}}
\usepackage[linesnumbered,ruled,vlined]{algorithm2e}
\DeclareMathOperator*{\argmin}{arg\,min}

\theoremstyle{plain}
\newtheorem{theorem}{Theorem}[chapter]
\theoremstyle{definition}
\newtheorem{corollary}{Corollary}[theorem]

\theoremstyle{definition}
\newtheorem{definition}{Definition}[chapter]
\usepackage{thmtools}
\usepackage{microtype}
\usepackage{imakeidx}
\makeindex
\usepackage{tgpagella}
\usepackage{caption}
\DeclareCaptionFont{10pt}{\fontsize{10pt}{12pt}\selectfont}
\RequirePackage{natbib}
\renewcommand\cite{\citep}	% to get "(Author Year)" with natbib    
\RequirePackage{hyperref}
\usepackage{xcolor}		% make links dark blue
\definecolor{darkblue}{rgb}{0, 0, 0.5}
\hypersetup{colorlinks=true,citecolor=darkblue, linkcolor=darkblue, urlcolor=darkblue}

\def\Snospace~{\S{}}

\definecolor{Gray}{gray}{0.85}
\definecolor{LightCyan}{rgb}{0.88,1,1}
\newcolumntype{a}{>{\columncolor{Gray}}c}
\newcolumntype{b}{>{\columncolor{white}}c}

\title {Generalized Optimal Linear Orders: \\
Theoretical Understanding of Word Order \\
Improves Empirical NLP}
\author {Rishi Bommasani}
\conferraldate {August}{2020}
\degreefield {M.S.}
\copyrightholder{Rishi Bommasani}
\copyrightyear{2020}

\begin{document}

\maketitle
\makecopyright

\begin{abstract}
The sequential structure of language, and the order of words in a sentence specifically, plays a central role in human language processing. Consequently, in designing computational models of language, the \textit{de facto} approach is to present sentences to machines with the words ordered in the same order as in the original human-authored sentence. The very essence of this work is to question the implicit assumption that this is desirable and inject theoretical soundness into the consideration of word order in natural language processing. In this thesis, we begin by uniting the disparate treatments of word order in cognitive science, psycholinguistics, computational linguistics, and natural language processing under a flexible algorithmic framework. We proceed to use this heterogeneous theoretical foundation as the basis for exploring new word orders with an undercurrent of psycholinguistic optimality. In particular, we focus on notions of dependency length minimization given the difficulties in human and computational language processing in handling long-distance dependencies. 
We then discuss algorithms for finding optimal word orders efficiently in spite of the combinatorial space of possibilities. We conclude by addressing the implications of these word orders on human language and their downstream impacts when integrated in computational models.
\end{abstract}

\begin{biosketch}
Rishi Bommasani was born in Red Bank, New Jersey and raised in Marlboro, New Jersey. Rishi received his B.A. from Cornell University with degrees in Computer Science and in Mathematics. He graduated \emph{magna cum laude} with distinction in all subjects. He continued studying at Cornell University to pursue a M.S. in Computer Science and was advised by Professor Claire Cardie. During his time as an undergraduate and M.S. student, Rishi received several awards including the Computer Science Prize for Academic Excellence and Leadership and multiple Outstanding Teaching Assistant Awards. He has been fortunate to have completed two internships at Mozilla in the Emerging Technologies team under the advisement of Dr. Kelly Davis in the DeepSpeech group. In his first summer at Mozilla, his work considered genuinely abstractive summarization systems; in his second summer, his research centered on interpreting pretrained contextualized representations via reductions to static embeddings as well as social biases encoded within these representations. He has been a strong advocate for advancing undergraduate research opportunities in computer science and was the primary organizer of numerous undergraduate reading groups, the first Cornell Days Event for Computer Science, and the inaugural Cornell Computer Science Research Night as well as its subsequent iterations. Rishi has been graciously supported by a NAACL Student Volunteer Award, ACL Student Scholarship, Mozilla Travel Grant, and NeurIPS Student Travel Grant. He will begin his PhD at Stanford University in the Computer Science Department and the Natural Language Processing group in Fall 2020. His PhD studies will be funded by a NSF Graduate Research Fellowship.
\end{biosketch}
\begin{dedication}

\emph{To my adviser, Claire, for your unrelenting support and unwavering confidence. \\
You will forever be my inspiration as a researcher and computer scientist. \\
In loving memory of Marseille. } 

\end{dedication}
\begin{acknowledgements}
There are many people to whom I am grateful and without whom the thesis would
have been almost impossible to write (much less finish):\footnote{These words are also the first words of Claire's thesis.} \par 
My adviser, Claire Cardie, has shaped who I am as a researcher and computer scientist with seemingly effortless grace. There is truly no way for me to compress my gratitude for her into a few words here. In part, I must thank her for putting up with my constant flow of ideas and for having the patience to allow me to learn from my own errant ideas and mistakes. She has truly adapted herself to accommodate my research interests and it is exactly that freedom that permitted me to develop this very thesis. She occasionally (quite kindly) remarked that she ``will be lost'' when I leave but it is really me who will be lost without her. She has set an unimaginably high standard for both my PhD adviser(s) and myself to match in the future. \par
I am also thankful for Bobby Kleinberg for the many hats he has worn (one of which was as my minor adviser). While there are countless encounters and small nuances I have learned from him well beyond algorithms, I think I will always hope to match his relentless curiosity and desire to learn. And I would like to thank Marty van Schijndel as his arrival at Cornell NLP has drastically changed how I view language, psycholinguistics, computational linguistics, and NLP. There has yet to be a dull moment in any of our interactions. \par 
I am deeply fortunate to have learned from and been guided by three rising stars --- Vlad Niculae, Arzoo Katiyar, and Xanda Schofield. As three remarkable young professors, I am quite lucky to have been one of their first students. What they might not know is that their theses were also quite inspiring in writing my own; thanks for that as well. In similar spirit, Kelly Davis has been a fabulous adviser during my two summers at Mozilla and I truly appreciate his willingness to let me explore and work on problems that I proposed. Thanks to Steven Wu for being a patient and insightful collaborator as well. \par 
Cornell NLP has been my home for the past few years and I must thank Lillian Lee for the role she played many years prior to my arriving in helping build this exceptional group with Claire. Many of her papers from the late 90's and early 2000's are my exact inspiration for writing well-executed research papers; her characteristic and constant insightfulness is simply sublime. I must also especially note Yoav Artzi, who as a researcher and a friend has deeply inspired my work and my commitment to being disciplined and principled. Cristian Danescu-Niculescu-Mizil, Sasha Rush, David Mimno, and Mats Rooth have been great members at the weekly NLP seminar and have further broadened the set of diverse perspectives towards NLP that I was privy to, further enriching me as a young scholar. More recently, the \texttt{C.Psyd} group has become an exciting community for me to  properly face the complexities of language and the intriguing perspectives afforded by psycholinguistics.  

At a broader scale, Cornell CS has been truly formative in how I view the world. I hope I will do Bob Constable proud in viewing the world computationally. I am very grateful to Kavita Bala for her untiring efforts to make the department a positive community that supports learning. And I am thankful to Joe Halpern and Eva Tardos for being excellent all-encompassing role models of what it means to be a great computer scientist and great faculty member. Similarly, Anne Bracy and Daisy Fan have been especially superlative for me in exemplifying great teaching. Adrian Sampson, Lorenzo Alvisi, Eshan Chattopadhyay, and countless others have all shown me the warm and collegial spirit that radiates throughout our department. I hope to carry this forward to the next places along my journey. Too often underappreciated, Vanessa Maley, Becky Stewart, Nicole Roy, Ryan Marchenese, and Randy Hess were all great resources that made my journey as an undergrad and masters student that much easier. \par
Ravi Ramakrishna is the person who showed me how exciting research can truly be and reignited my passion for mathematics. He might not know it, but counterfactually without him and the environment he fostered in MATH 2230, it is hard to imagine me being where I am now.  \par
But that is enough with acknowledging faculty and old folks. I have been very fortunate to have a great number of research friends in Cornell NLP and across Cornell who have mentored me and been great to learn alongside:\par
\noindent Esin Durmus, Ge Gao, Forrest Davis, Tianze Shi, Maria Antoniak, Jack Hessel, Xilun Chen, Xinya Du, Kai Sun, Ryan Benmalek, Laure Thompson, Max Grusky, Alane Suhr, Ana Smith, Justine Zhang, Greg Yauney, Liye Fu, Jonathan Chang, Nori Kojima, Andrea Hummel, Jacob Collard, Matt Milano, Malcolm Bare. \par
Further, I have had many wonderful friends who have encouraged me, especially Linus Setiabrata, Janice Chan\footnote{I am also grateful to Janice for proofreading parts of this thesis.}, Avani Bhargava, Isay Katsman, Tianyi Zhang, Cosmo Viola, and Will Gao. I have also cherished my time with: \\
\noindent Andy Zhang, Eric Feng, Jill Wu, Jerry Qu, Haram Kim, Kevin Luo, Dan Glus, Sam Ringel, Maria Sam, Zach Brody, Tjaden Hess, Horace He, Kabir Kapoor, Yizhou Yu, Rachel Shim, Nancy Sun, Jacob Markin, Harry Goldstein, Chris Colen, Ayush Mittal, Cynthia Rishi, Devin Lehmacher, Brett Clancy, Daniel Nosrati, Victoria Schneller, Jimmy Briggs, Irene Yoon, Abrahm Maga\~{n}a, Danny Qiu, Katie Borg, Katie Gioioso, Swathi Iyer, Florian Hartmann, Dave Connelly, Sasha Badov, Sourabh Chakraborty, Daniel Galaragga, Qian Huang, Judy Huang, Keely Wan, Amrit Amar, Daniel Weber, Ji Hun Kim, Victor Butoi, Priya Srikumar, Caleb Koch, Shantanu Gore, Grant Storey, Jialu Li, Frank Li, Seraphina Lee. 

Throughout my time at Cornell CS, two sources of persistent inspiration were Rediet Abebe and Jehron Petty. It was truly remarkable to witness the change they drove in the department, and beyond, while I was there. 

I am grateful to all of the students who I have TA-d for for helping me grow as a teacher. Of special note are the students of CS 4740 in Fall 2019 when I co-taught the course with Claire; I appreciated their patience in tolerating my first attempts to prepare lectures for a course.\footnote{My intent is for this thesis to be understandable to any student who has completed CS 4740.} Similarly, I have been extremely privileged to have worked with and advised a number of exceptional undergraduates and masters students. I hope that I have helped them grow as researchers and better appreciate the exciting challenges involved in pursuing NLP/computational linguistics research: \\
Aga Koc, Albert Tsao, Anna Huang, Anusha Nambiar, Joseph Kihang'a, Julie Phan, Quintessa Qiao, Sabhya Chhabria, Wenyi Guo, Ye Jiang. 

As I prepare for the next step in my academic journey as a PhD student in the Stanford Computer Science Department and Stanford NLP group, I would like to thank a number of faculty, current (or recently graduated) PhD students, and members of my graduate cohort who helped me during the decision process: \\
\textit{Faculty}: Percy Liang\footnote{Percy's own masters thesis at MIT was quite influential in writing/formatting this thesis.}, Dan Klein, Tatsu Hashimoto, Dan Jurafsky, Chris Potts, Chris Manning, John Duchi, Jacob Steinhardt, Noah Smith, Yejin Choi, Jason Eisner, Ben van Durme, Tal Linzen, Graham Neubig, Emma Strubell, Zach Lipton, Danqi Chen\footnote{This section was inspired by Danqi's own dissertation.}, Karthik Narasimhan. \\
\textit{PhD students during the process}: Nelson Liu, John Hewitt, Pang Wei Koh, Urvashi Khandelwal, Aditi Raghunathan, Robin Jia, Shiori Sagawa, Kawin Ethayarajh, Eva Portelance, Sidd Karamcheti, Nick Tomlin, Eric Wallace, Cathy Chen, Sabrina Mielke, Adam Poliak, Tim Vieira, Ryan Cotterell, Ofir Press, Sofia Serrano, Victor Zhong, Julian Michael, Divyansh Kaushik. \\
\textit{2020 PhD admits}: Alisa Liu, Han Guo, Suchin Gururangan, Katherine Lee, Lisa Li, Xikun Zhang, Aditya Gupta, Victor Sanh, Mengzhou Xia, Megha Srivastava. 

A special thank you is also due to those who helped organize virtual visits in light of the unique challenges posed by the COVID-19 pandemic that spring. 

Conference travel to present my research was funded by a NAACL Student Volunteer Award, ACL Student Scholarship, Mozilla Travel Grant, and NeurIPS Student Travel Grant in addition to funding from Cornell University and Claire. \\
\newpage 
\noindent The final thank you must go to my parents, Ram and Saila Bommasani, for their patience to allow me to explore what made me happy and their enduring encouragement in allowing me to forge my own path. Few parents understand these subtleties of parenting better than you.

\end{acknowledgements}

\contentspage
\tablelistpage
\figurelistpage
\listofalgorithms

\normalspacing \setcounter{page}{1} \pagenumbering{arabic}
\pagestyle{cornell} \addtolength{\parskip}{0.5\baselineskip}
\chapter{Introduction}\label{chapter:introduction}
In this chapter, we set forth the motivations and contributions of this work.  
\section{Motivation}
Natural language plays a critical role in the arsenal of mechanisms that humans use to communicate. Inherently, natural language is a rich code with fascinating linguistic structure that humans rely upon to transfer information and achieve communicative goals \citep{informationtheory, miller1951language, chomsky-anti-statistics, chomskysyntax, hockett1960, greenberg1963, chomsky1986knowledge, pinker_bloom_1990, hawkins1994, pinker2003, pinker2005, pinker2007, jaeger2011, chomsky2014aspects, chomsky2014minimalist, gibson2019}. In spite of the fact that natural language is fundamentally a mechanism for human-human discourse, in recent years we have witnessed the emergence of potent computational models of natural language. In particular, society as a whole has come to rely on a variety of language technologies. Prominent examples include machine translation \citep{weaver49, shannonweaver63, lopez2008, koehn2010, googlenmt}, speech recognition and synthesis \citep{vocoder, speechsynthesis, asr2014, neuralasr, neuralspeechsynthesis}, information retrieval and search \citep{termfrequency, salton67, salton71, inversedocumentfrequency, salton75, saltonvectorspace, salton-ir, tf-idf, pagerank, googlesearch2005, manning2008, google2009, googlesearch+bert}, large-scale information extraction \citep{andersen1992, muc3, muc6, cardie97, califf97,  wilks1997, gaizauskas98, web-scale-ie,  choi2005, openie, distant-rel-ex, openie2, babelnet, piskorski2013}, and sentiment analysis \citep{textclassificationsentiment, classificationmutualinformation, subj, mr, large-scale+sentiment-analysis, pang2008, bautin2008international, ye2009, asur2010, liusentiment2012, chau2012, li2014, ravi2015, xing2017}. And the scope for language technologies is only projected to grow even larger in the coming years \citep{hirschberg-manning-2015}. \\

In designing computational models of language, a natural consideration is specifying the appropriate algorithmic primitives. Classical approaches to algorithm design have been considered but generally have struggled to model language faithfully; the exacting nature of deterministic algorithms like quick-sort is ill-suited to the myriad ambiguities found within natural language. Based on empirical findings, the field of natural language processing (NLP) has drifted towards machine learning and probabilistic methods \citep{charniak-1993, manning+schutze, jurafsky+martin, steedman2008, hirschberg-manning-2015, goldberg2017, eisenstein2019, mcclelland2019} despite the initial dismissal of such statistical approaches by \citet{chomsky-anti-statistics}. However, this transition alone does not reconcile that the mathematical primitives used in machine learning and deep learning \citep{machinelearningtextbook, machinelearningtextbook2, deep-learning-textbook}, i.e.~vectors, affine transformations, and nonlinearities, are inconsistent with those present in natural language, i.e.~characters, words, sentences. One of the characteristic successes of NLP in the past decade has been the development of word embeddings \citep{bengio2003, collobert2008, collobert2011, word2vec, glove, wordembeddingsthesis}: explicit methods for encoding words as vectors where the abstract semantic similarity between words is codified as concrete geometric similarity between vectors. In general, a hallmark of modern NLP is the inherent tension between linguistic representations and computational representations as, simply put, words are not numbers. \\

In this thesis, we study computational representations of a fundamental aspect of language: word order. Within natural language, sentences are a standard unit of analysis\footnote{The annual CUNY conference, now in its $34^{th}$ iteration, is entirely dedicated to the topic of sentence processing.} and every sentence is itself a sequence of words. The central question that we consider in this thesis is whether the order of words in a sentence, ascribed by the human who produced it, is the appropriate order for computational models that attempt to comprehend the sentence (in order to perform some downstream task). In order to make principled progress towards answering this question, we contextualize our work against the backdrop of considerations of word order/linear order in the literature bodies of psycholinguistics and algorithms. From a psycholinguistic standpoint, the word order already attested by natural language sentences can be argued to be indicative of an optimization to facilitate human processing. Simultaneously, from an algorithmic perspective, word orders observed in natural language may be computationally suboptimal with respect to certain combinatorial objectives, which naturally begs the question of how (computational) processing may change when presented with optimal word orders. In this sense, the unifying approach we adopt in this thesis is to interlace motivating prior work from both psycholinguistics and algorithms to specify novel word orders, which we then evaluate empirically for downstream NLP tasks. 
\section{Contributions}
\noindent \textbf{Generalized Optimal Linear Orders.} The central contribution of this work is a framework for constructing novel word orders via an optimization procedure and, thereafter, studying the impacts of these orders on downstream NLP tasks. Consequently, we begin by extending and connecting previously disconnected literature from the algorithms community with work that focuses on modelling word order in NLP. We also present three novel word orders generated via the \texttt{Transposition Monte Carlo} algorithm that we introduce. These orders rely on a simple greedy heuristic that allows for (somewhat-transparent) balancing of the original sentence's word order, therefore preserving information encoded in the original word order, and optimization against the objectives we introduce. We demonstrate how to incorporate these novel word orders, which are optimal (with respect to a combinatorial objective), with downstream NLP. In particular, we propose the \texttt{pretrain-permute-finetune} framework, which seamlessly integrates our novel orders with large-scale pretraining. We empirically evaluate the benefits of our method and show it can yield improvements for English language text classification tasks.\\ 

\noindent \textbf{Quantified (sub)optimality of natural language.} Due to the explicit computational framework we develop, we can further quantify the extent to which various natural languages are suboptimal with respect to objectives related to dependency length minimization. As we discuss subsequently, there has been significant work in the psycholinguistics community towards demonstrating that human languages are effective at dependency minimizing (compared to random word orders) and our work helps provide the dual perspective by clarifying the extent to which they are suboptimal. \\ 

\noindent \textbf{Survey of word order in language processing.} Research in human language processing, and sentence processing specifically, has a rich history of studying the influence of word order on processing capabilities in humans. While the corresponding study in natural language processing has arguably lacked similar rigor, this thesis puts forth a joint summary of how multiple communities have studied word order in language processing.
%to extend the long-standing literature body for human language processing. 

\section{Organizational Outline}
The remainder of this thesis is organized as follows. \\ 

We begin in Chapter 2 (\autoref{chapter:background}) by introducing fundamental preliminaries. These include a self-contained introduction to dependency grammars as well as a discussion of the disparate treatments of word order within NLP. We then examine some of the literature on studying word order in human languages in Chapter 3 (\autoref{chapter:wordorderinhlp}), with a specific focus on cognitive and psycholinguistic arguments centered on human language processing. We pay special attention to the line of work focused on dependency length minimization and dependency locality effects (\autoref{sec:dependencylengthminimization}). \\ 

In Chapter 4 (\autoref{chapter:algorithmic}), we shift gears by providing a generalized framework for studying notions of optimality with respect to dependency length and word order. We further provide several algorithms introduced in prior work that permit tractable (polynomial-time) optimization of various combinatorial objectives. We augment these with heuristic algorithms that allow for balance between retaining the original sentence's order and purely optimizing objectives related to dependency parses. \\
% With this framework and these algorithms in hand, we study the (sub)optimality of natural languages in Chapter 5 (\autoref{chapter:optimallinearordersandNL}). 
% We discover \textbf{\color{red}{FILL IN\dots}} 

 In Chapter 5 (\autoref{chapter:optimallinearordersforNLP}), we consider how the novel word orders we have constructed influence dependency-related costs and downstream performance in NLP. We find that English already substantially optimizes for the objectives we study compared to a random word order baseline. Further, we show that there is still a substantial margin for further optimization over English and that the heuristic algorithms we introduce perform slightly worse than algorithms that are established in the literature from an optimization perspective. Intriguingly, we find that optimizing for some objectives (most notably \textsc{minimum linear arrangement}) can yield to improvements on other objectives but does not in all cases (especially for the \textsc{bandwidth} objective).  Given these observations, we then evaluate on downstream text classification tasks. We find that the standard English order is a strong baseline but can be improved over in four of the five datasets we study (by using a novel word order introduced in this work). In particular, we show that word orders generated by our heuristic approach often outperform those generated by standards algorithms, suggesting that word order design that strictly optimizes combinatorial objectives is arguably naive and may not be sufficient/desirable for modelling natural language.  \\ 

We conclude this thesis by providing a contextualized summary of the results in Chapter 6 (\autoref{chapter:conclusions}). We further provide a discussion of open problems, future directions, and broader lessons. We complement this with a transparent reporting of the inherent limitations of this work. \\

In \autoref{appendix:reproducibility}, we provide an exhaustive set of details to fully reproduce this work. This includes references to code we used to conduct all experiments and generate all tables/figures used in this work. We further provide details for accessing all datasets used in the work. In \autoref{chapter:appendix-additional-results}, we provide additional results that we did not include in the main thesis. These results help clarify the performance of models for suboptimal hyperparameter settings (and, implicitly, the stability of the results to various hyperparameter settings).  
\section{Previous Works}
The underlying foundation for this work was originally published in \citet{long-distance-dependencies-don't-have-to-be-long}, which was presented at \textit{ACL 2019} during the main conference in the \textit{Student Research Workshop}. It was further presented to a machine learning audience at \textit{NeurIPS 2019} in the \textit{Context and Compositionality in Biological and Artificial Neural Systems Workshop}. In both past works, part of the content that appears in \autoref{chapter:algorithmic} and \autoref{chapter:optimallinearordersforNLP} was introduced. The remainder of the thesis was specifically created for the purpose of this thesis. The official version of this thesis is published in the Cornell University Library.\footnote{See \url{https://ecommons.cornell.edu/handle/1813/103195}.}
\chapter{Background}\label{chapter:background}
In this chapter we introduce preliminary machinery that we will use throughout the thesis --- the dependency parse --- and the existing treatments of word order in NLP. 
\section{Primitives}\label{sec:primitives}
In this thesis, we will denote a sentence by $\sentence$ which is alternatively denoted by a sequence of words $\langle w_1 \dots w_n \rangle$. For simplicity of prose, we will assume sentences contain no duplicates though none of the algorithms or results we present make use of this assumption. Given a sentence, the task of decomposing it into its corresponding sequence of words is known as \textit{tokenization}.\footnote{In this thesis, we will make no distinction between the terms \textit{word} and \textit{token}. Similarly, we will not distinguish \textit{word types} (lexical categories) from \textit{word tokens} (individual occurrences of word types).} In practice, while tokenization technically describes breaking ``natural language text [...]
into distinct meaningful units (or tokens)'' \cite{kaplan2005tokenization}, it is often conflated with various text/string normalization processes (e.g.~lowercasing). \\

In `separating' languages, such as English, the use of whitespace can be taken as a reasonably proxy for token boundaries whereas in other languages, such as Mandarin Chinese, this is not feasible. In general, in this work we will assume access to a tokenizer for the language being studied and will not reconsider any errors introduced during tokenization. In particular, while tokenization is not strictly solved \cite{tokenizationsurvey}, high-quality tokenizers exist for a variety of languages, including some low-resource languages, in standard packages such as Stanford CoreNLP \citep{corenlp} and {\gillius Stanza} \citep{stanza}. 
\section{Dependency Grammars}\label{sec:dependencygrammars}
In this work, we consider \textit{syntactic} representations of language. Specifically, we focus our efforts on \textit{dependency grammars}, which were first formalized in the modern sense by Lucien Tesni{\`e}re \cite{dependencygrammars}.\footnote{\citet{dependencysurvey} provides a more comprehensive primer on dependency grammars and dependency parsing.} Under a dependency grammar, every sentence has a corresponding \textit{dependency parse} which encodes binary relations between words that mark syntactic dependencies. This approach for specifying a sentence-level syntactic structure differs greatly from the phrase-based/constituency grammars championed by Leonard Bloomfield and Noam Chomsky \citep{bloomfieldsyntax, chomskysyntax}. The central difference rests on how clauses are handled: phrase-structure grammars split clauses into subject noun phrases and predicate verb phrases whereas dependency grammars are verb-focused. Further, phrase-structure grammar may generate nodes that do not correspond to any single word in the sentence.\\

Formally, a dependency grammar attributes a dependency parse $\mathcal{G}_{\sentence}$ to every sentence $\sentence = \langle w_1 \dots w_n \rangle$ where $\mathcal{G}_{\sentence}$ is a directed graph with vertex set \hbox{$\mathcal{V} = \big\{w_i \mid i \in [n]\big\}$} and edge set $\mathcal{E}_{\ell}$ given by the directed binary dependency relations. Each dependency relation is labelled (hence a dependency parse is an edge-labelled directed graph) and the specific labels are based on the specific \textit{dependency formalism} used, which we describe subsequently. The direction of the edges is from the syntactic head (the source of the edge) to the syntactic child (the target of the edge); the head is of greater syntactic prominence/salience than the child. Dependency parses are constrained to be trees and, since the main verb plays a central role, are often conceived as rooted trees that are rooted at the main verb. \\
\begin{figure}
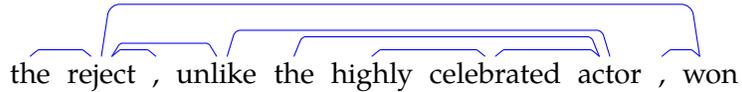

  \centering
  \small
  \begin{dependency}[hide label, edge unit distance=.5ex]
    \begin{deptext}[column sep=0.05cm]
      the\& reject\& ,\& unlike\& the\& highly \& celebrated\& actor\& ,\& won \\
    \end{deptext}                                                                                                                                                                                                                           
    \depedge[edge style={blue}, edge above]{2}{1}{.}
    \depedge[edge style={blue}, edge above]{2}{3}{.}
    \depedge[edge style={blue}, edge above]{2}{4}{.}
    \depedge[edge style={blue}, edge above]{4}{8}{.}
    \depedge[edge style={blue}, edge above]{7}{6}{.}
    \depedge[edge style={blue}, edge above]{8}{5}{.}
    \depedge[edge style={blue}, edge above]{8}{7}{.}
    \depedge[edge style={blue}, edge above]{10}{2}{.}
    \depedge[edge style={blue}, edge above]{10}{9}{.}
  \end{dependency}
    
\caption{Dependency parse of the given sentence. Dependency arcs are drawn canonically (above the linear sequence of words) and the sequence has been lowercased and dependency parsed using the \texttt{spaCy} parser \cite{spacy} for English.}
\label{fig:exampledependencyparse}
\end{figure}

In Figure~\ref{fig:exampledependencyparse}, we provide an example of a dependency parse for the given sentence.\footnote{We do not illustrate the direction or labels of any dependency relations. The reasons for doing so will be made clear in \autoref{chapter:algorithmic}.} As is shown, we will drawn dependency parses in this canonicalized format where all arcs are strictly above the linear sequence of words. If the dependency parse, when depicted this way, has no intersecting edges (i.e.~the drawing is a constructive proof that the underlying graph is planar), we call the dependency parse \textit{projective} \cite{haysprojective}. Under many theories for dependency grammars, most/all sentences in most/all languages are argued to satisfy projectivity constraints. In particular, violations of projectivity in English are very infrequent \cite{minlaprojective} and \citet{projectivityczech} estimates that in Czech, a language that is one of the most frequent to violate projectivity, non-projective sentences constitute less than $2\%$ of all sentences. We revisit the notion of projectivity, as it will prove to be useful for algorithms we subsequently study, in \autoref{subsec:projectivityconstraints}.
\subsection{Dependency Parsing}
In this work, we consider sentences that are both annotated and not annotated with a gold-standard dependency parse. When sentences are annotated, they are taken from the Universal Dependencies Treebank\footnote{{ \url{https://universaldependencies.org/}}} and were introduced by \citet{UD} with annotations following the Universal Dependencies dependency formalism. When sentences are not annotated, we parse them using off-the-shelf pretrained parsed that we describe in later sections. In particular, we strictly consider unannotated data for English. In English, there exist several high-quality pretrained parsers \cite{dozat2016, dozat2017, spacy, tianze2018} and dependency parsing is relatively mature. Comparatively, for other natural languages, and especially low-resource languages, off-the-shelf dependency parsing is less viable \cite{lowresourcedependency} and we revisit this in \autoref{sec:limitations}. 
\section{Word Order in Natural Language Processing}\label{sec:related-work}
In order to understand how word order should be modelled computationally, we begin by cataloguing the dominant approaches to word order in the NLP literature. We revisit the most pertinent methods more extensively in \autoref{chapter:optimallinearordersforNLP}. 
\subsection{Order-Agnostic Models}\label{subsec:sequentialmodels}
Given the challenges of modelling word order faithfully, several approaches in NLP to word order have entirely sacrificed modelling order to prioritize other pertinent phenomena. In some settings, where document-scale representations are desired, it has been argued that the nuances of word order within sentences is fairly marginal. Two well-studied regimes are the design of \textit{topic models} and \textit{word embeddings}. \\

\noindent \textbf{Topic Models.} Topic models are (probabilistic) generative models of text collections that posit that the texts are generated from a small set of latent \textit{topics}. This tradition of proposing generative models of text originates in information retrieval \citep{salton-ir} and has led to a series of works towards designing topic models that yield topics that well-aligned with human notions of topics. Almost all topic models represent documents by their bag-of-words representation, hence neglecting order. The most famous topic model is \textit{Latent Dirichlet Allocation} (LDA) \citep{lda}, which proposes a hierarchical Bayesian approach towards generative modelling; model parameters can be efficiently inferred via Markov Chain Monte Carlo \citep{ldamcmc}, variational inference \citep{lda, ldaonline}, Bayesian Poisson factorization \citep{ldapoisson, ldabayesianpoisson1, ldabayesianpoisson2}, and spectral methods using either method of moments \citep{ldaspectralmoments} or anchor words/separability assumptions \citep{ldaspectralanchor1, ldaspectralanchor2, ldaspectralanchor3}. Order-agnostic topic models have seen a wide array of applications in computational social science and the digital humanities; contributions have been made via textual analysis to the disciplines of political science \citep{ldapoliticalscience}, literature \citep{ldaliterature}, and history \citep{ldahistory} among several others. For a more extensive consideration of topic models, see \citet{topicmodelsurvey1, topicmodelsurvey2, ldathesis}. \\

\noindent \textbf{Word Embeddings.} Word embeddings \citep{bengio2003, collobert2008, collobert2011} are learned mappings from lexical items to vectors that encode natural language semantics in vector spaces. Most word embeddings hinge on a particular interpretation of the distributional hypothesis \citep{distributionalhypothesis1, distributionalhypothesis2}. Classical methods such as LSA \citep{lsa} factorized co-occurrence statistics whereas more recent neural methods \citep{word2vec, glove, fasttext} predict various co-occurrence statistics \citep{predict>count}. In both cases, most approaches are largely order-agnostic (or use low order n-gram statistics) and subsequent work has shown that many neural methods for word embedding can be re-interpreted as factorizations (of the pointwise mutual information) as well \citep{levyfactorize1, levyfactorize2, ethayarajhfactorize}. Similar to topic models, word embeddings have seen a wide array of applications not just within NLP (as initializations using pretrained word representations) but also beyond NLP including to study diachronic societal biases \citep{wordembeddingsdiachronic} and cultural associations \citep{wordembeddingsculture}. For a more extensive consideration of word embeddings, see \citet{wordembeddingssurvey, wordembeddingsthesis}.\\

\noindent \textbf{Bag-of-Words Classifiers.} While topic models and word embeddings learn representations from a collection of documents, bag-of-words and order-agnostic techniques have also been considered in building representations of language within a document and even within a sentence. Many of these methods are classical approaches to text classification. Initial approaches adapted standard algorithms from the machine learning community \citep{machinelearningtextbook} for linear and log-linear classification such as the Naive Bayes and maximum entropy algorithms \citep{naivebayesnlp1, maxent, naivebayesoptimal, naivebayesnlp2, naivebayesir, textclassificationsentiment}, whereas later works considered nonlinear classifiers such as SVMs \citep{textclassificationsvms, textclassificationsentiment} and feed-forward neural networks \citep{collobert2008, collobert2011}. Simultaneously, many of these works such as those of \citet{naivebayesir} and \citet{classificationmutualinformation} had their origins in information retrieval and information theory. In these works, it was standard to use order-agnostic \textit{term frequency} (TF) \citep{termfrequency} and \textit{inverse document frequency} (IDF) \citep{inversedocumentfrequency} features, commonly under the joint framing of the TF-IDF weighting schema \citep{tf-idf} as in the Rocchio classifier \citep{rocchio, rocchio1997}. Comprehensive surveys and analyses of models for text classification are provided by \citet{textclassificationsurvey1, textclassificationsurvey2, textclassificationsurvey3, textclassificationsurvey4}. \\

\noindent \textbf{Order-Agnostic Sentence Encoders.} Following the introduction of Word2Vec and neural networks in NLP in the early 2010's, the community gravitated towards deep learning approaches that no longer required explicitly feature engineering. Consequently, order-agnostic approaches within sentences became less frequent. Nonetheless, order-agnostic representation learning over word representations for sentence encoding has proven to be effective as a strong (and cheap) baseline. Learning-based order-agnostic sentence encoding often uses variants of deep averaging networks for text classification tasks \citep{dan, adversarial-dan}. However, subsequent work showed that the deep averaging was unnecessary and that simple averaging was sufficient \citep{sentenceaverage}. Additionally, some works have viewed averaging methods theoretically (often as random walks) \citep{arora2016, ethayarajh2018} and different weighting schema have emerged to better encode the fact that word-level representations do not contribute uniformly towards the meaning of a sequence \citep{sif}. 

\subsection{Sequential Models}\label{subsec:sequentialmodels}
Given that natural language has an explicit sequential structure and this structure is informative (hence our interest in word order), a large family of approaches in NLP have attempted to model the sequential nature directly. \\

\noindent \textbf{Markov Models.} Markov models are a family of statistical models which make \textit{Markovian assumptions} --- assumptions that strictly  bound the length of dependencies that can be modelled. In particular, a Markov model of Markov order $n$ cannot model a distance of length at least $n+1$ directly. Nonetheless, a recent line of theoretical results suggest that there are workarounds for modelling long-distance dependencies in such models 
\citep{sharan2017, sharan2018}. Within NLP, hidden Markov models (HMMs) have been used for a variety of sequence-tagging applications including part-of-speech tagging, named entity recognition, and information extraction \citep{jelinek1976hmm, freitag99hmm, freitag2000hmm, toutanova2002hmm, collins2002hmm}. In using HMMs in NLP, the causal factorization of the desired probabilities is generally estimated using n-gram statistics. In maximum entropy Markov models (MEMMs), a maximum entropy classifier is introduced to add expressiveness and this has been shown to be more effective in most settings \citep{lau1993memm, ratnaparkhi1994memm, ratnaparkhi1996memm, reynar1997memm, toutanova2000memm, mccallum2000memm}. Alternatively, conditional random fields (CRFs) proved to be effective in weakening the strong independence assumptions that are built into HMMs and the biases\footnote{Towards states that had few successors.} that are inherent to MEMMs \citep{lafferty2001crf, sha2003crf, pinto2003crf, roark2004crf, peng2004crf, sutton2007crf, sutton2012crf}. \\

\noindent \textbf{Parsing.} Sequence-tagging problems, which were extensively studied using Markov models, are a special case of \textit{structured prediction} problems that are prevalent in NLP. In the well-studied setting of parsing, whether it was syntactic constituency parsing, syntactic dependency parsing, or semantic parsing, several approaches were taken to jointly model the structure of the parse and the sequential structure of language \citep{kay-1967-experiments, earley-parser, charniak1983parser, pereira-warren-1983-parsing, kay1986parsing,  kay-1989-head, eisner-1996-three, collins-1996-new, collins-1997-three, charniak-etal-1998-edge, gildea2002, collins-2003-head, klein-manning-2003-accurate, klein-manning-2003-parsing, taskar2004, mcdonald-etal-2005-non, mcdonald-pereira-2006-online, chenmanning2014, dozat2016, dozat2017, shi2017-fast, shi2017-global, gomez2018-global, tianze2018, shi2020}. When compared to other settings where sequential modelling is required in NLP, parsing often invokes highly-specialized routines that center on the unique and rich structure involved. \\

\noindent \textbf{Recurrent Neural Networks.} Given the cognitive motivations for modelling language sequentially in computational methods, \citet{elmanrnn} pioneered the use of recurrent neural networks (RNNs). While these networks have a connectionist interpretation \citep{rumelhart1986, jordan1989}, they ultimately proved to be ineffective due to technical challenges such as vanishing/exploding gradients in representing long-distance relationships. Consequently, later works introduced gated networks such as the long short-term memory (LSTM) network \citep{lstm}. Analogous to the dramatic performance improvements experienced due to word embeddings such as Word2Vec, the community observed similarly benefits in the early to mid 2010's due to LSTMs. This prompted further inquiry into a variety of RNN variants \citep[e.g.][]{gru, qrnn, sru, mogrifier}. More recently, a line of theoretical works has worked towards classifying the theoretical differences between these variants \citep{cnn-rnn-wfsm, weiss2018, rational-recurrence, suzgun2019-generalization, suzgun2019-dynamic, merrill2019, neural-fst}. This has recently culminated in the work of \citet{formalrnnhierarchy} which establishes a formal taxonomy that resolves the relationship between various RNN varieties and other methods from classical automata theory such as weighted finite state machines. \\

\noindent \textbf{Attention.} The emergence of neural networks in NLP for sequence modelling naturally led to their adoption in natural language generation tasks such as machine translation \citep{kalchbrenner2013, cho2014, seq2seq} and summarization \citep{rush2015, chopra2016, nallapati2016}. In these settings, attention came to be a prominent modelling innovation to help induce alignment between the source and target sequences \citep{additiveattention, multiplicativeattention}. Since then, attention has seen application in many other settings that involve sequential modelling in NLP as it enables networks to model long-distance dependencies that would be otherwise difficult to model due to the sequential/recurrent nature of the networks. Given attention's widespread adoption, a line of work has been dedicated to adding sparsity and structure to attention \citep{martins2016, structured-attention-network, niculae2017, mensch2018, malaviya2018, peters2018, niculae2018learning, correia2019, peters2019} whereas a separate line of work has studied its potential utility as an interpretability tool for explaining model behavior \citep{attention-not-explanation, attention-interpretation, supervised-attention-explanation, attention-not-not-explanation, deceiving-attention}. \\

\subsection{Position-Aware Models}\label{subsec:position-aware-models}
Sequential models directly model the sequential nature of language. In recent years, there has been an emergence and considerable shift towards using position-aware models/set encoders. In particular, these models implicitly choose to represent a sequence $\langle w_1 \dots w_n \rangle$ as the set $\{(w_i, i) \mid i \in [n]\}$\footnote{The correspondence between arbitrary sequences and sets of this structure is bijective} as was described in \citet{sequencessets}. In this sense, the encoder is aware of the position but does not explicitly model order (e.g.~there is no explicit notion of adjacency or contiguous spans in this encoding process). Early works in relation extraction also considered position-aware representations \citep{position-aware-relex}.\footnote{To the author's knowledge, this observation and citing of \citet{sequencessets} and \citet{position-aware-relex} has been entirely neglected in all past works in the NLP community \citep[c.f.][]{transformers, transformer-xl}.}\\

\noindent \textbf{Transformers.} \citet{transformers} introduced the Transformer architecture, which has become the dominant position-aware architecture in modern NLP. In particular, all sequences are split into $512$ subword units and subwords are assigned lexical embeddings and position embeddings, which are then summed to yield non-contextual subword representations. These $512$ subword vectors are then iteratively passed through a series of Transformer layers, which decompose into a self-attentive layer\footnote{Self-attention is attention in the sense of \citet{additiveattention} where the given sequence is used in both roles in the attention computation.} and a feed-forward layer. Since these operations are fully parallelizable, as they have no sequential dependence, large-scale training of Transformers on GPU/TPU computing resources has propelled performance forward on a number of tasks. Similarly, since these models can compute on more data per unit time than sequential models like LSTMS\footnote{Given the constraints of current hardware.}, they have led to a series of massive pretrained models that include: GPT \citep{gpt}, BERT \citep{bert}, GPT-2 \citep{gpt2}, XLNet \citep{xlnet}, RoBERTa \citep{roberta}. SpanBERT \citep{spanbert}, ELECTRA \citep{electra}, ALBERT \citep{albert} and T5 \citep{t5}. \\

\noindent \textbf{Position Representations.} Given that what differentiates position-aware models from order-agnostic models is their position representations, surprisingly little work has considered these representations \citep{positionembeddings}. In the original Transformer paper, position embeddings were frozen using cosine waves to initialize them. Recent work has put forth alternative approaches for encoding position \citep{discretecosine}. In particular, \citet{complexposition} demonstrate that using complex-valued vectors, where the amplitude corresponds to the lexical identity and the periodicity corresponds to the variation in position, can be a principled theoretical approach for better modelling word order in Transformer models. Separately, \citet{relative-position} and \citet{transformer-xl} argue for encoding position in a relative fashion to accommodate modelling longer sequences (as the standard Transformer is constrained to $512$ positions).   

\subsection{Alternative Word Orders}\label{subsec:alternative}
Given that natural language processing tasks often requiring understanding an input text, it is unsurprising that most works which model the input in an order-dependent way (generally implicitly) choose to specify the word order to be the same as the order already given in the input. A frequent exception is bidirectional models, which have seen applications in a number of settings. Beyond this, other approaches have considered alternative word orders as a mechanism for studying alignment between different sequences. Much of this literature has centered on machine translation. \\

\noindent \textbf{Bidirectional Models.} 
One natural choice for an alternative order is to use the reverse of the order given. For a language such as English which is read from left-to-right, this would mean the order given by reading the input sequence from right-to-left. While some works have compared between left-to-right models and right-to-left models \citep{seq2seq}, in most downstream settings, bidirectional models are preferred. A bidirectional model is simply one that integrates both the left-to-right and right-to-left models; the bidirectional RNN is a classic model of this type \citep{bidirnn}. Shallowly bidirectional models do this by independently modelling the input from left-to-right and right-to-left and subsequently combining (generally by concatenation or vector addition) the resulting output representations. Such approaches have seen widespread application in NLP;  the ELMo  pretrained model is trained in a shallowly bidirectional fashion \citep{elmo}. In comparison, with the emergence of Transformers, it is possible to process part of the input (e.g.~a single token) while conditioning on the entirety of the remainder of the input at once. Such models are often referred to as deeply bidirectional; BERT \citep{bert} is a model pretrained in this way by making use of a denoising objective in masked language modelling\footnote{Masked language modelling is a cloze task where the objective is to predict the masked word in the input sequence conditional on the remainder of the sequence, which is unmasked.} as opposed to the standard causal language modelling used in ELMo. \\

\noindent \textbf{Permutation Models.} From a language modelling perspective, a unidirectional left-to-right (causal) language model factorizes the sequence probability $p(\langle w_1 \dots w_n \rangle)$ as 
\begin{equation}
    p\left(\langle w_1 \dots w_n \rangle\right) = \prod_{i=1}^np\left(w_i \mid \langle w_1 \dots w_{i-1} \rangle\right). 
\end{equation}
In comparison, a unidirectional right-to-left language model factorizes the sequence probability as
\begin{equation}
    p\left(\langle w_1 \dots w_n \rangle\right) = \prod_{i=1}^np\left(w_i \mid \langle w_{i+1} \dots w_{n} \rangle\right). 
\end{equation}
In the recent work of \citet{xlnet}, the authors introduce a strikingly new approach which generalizes this perspective. In particular, any given ordering of the sequence $\langle w_1 \dots w_n \rangle$ corresponds to a unique factorization of this sequence probability. In their model, XLNet, the authors sample factorizations uniformly (hence considering the behavior in expectation across all $n!$ possible factorizations) and, alongside other modelling innovations, demonstrate that this can be effective in language modelling. As we will demonstrate, our approach could be seen adopting the perspective of trying to identify a single optimal order than sampling from all possible orders with equal likelihood. \\

\noindent \textbf{Order Alignment.} For tasks that involve multiple sequences, order plays an additional role of facilitating (or inhibiting) alignment between the difference sequences. In machine translation, the notion of alignment between the source and target languages is particularly important.\footnote{In fact, attention \citep{additiveattention, multiplicativeattention} emerged in machine translation precisely for the purpose of better aligning the fixed input sequence and the autoregressively generated output sequence.} As a consequence, two sets of approaches towards ensuring improved alignment (explicitly) are \textit{preorders} (changing the order of the source language input to resemble the target language) and \textit{postorders} (changing the order of a monotone output translation to resemble the target language). 
\begin{itemize}
    \item \textit{Preorders} --- Preorders have been well-studied in several machine translation settings. In particular, preorders have been designed using handcrafted rules \citep{brown1992, collins2005preorder, wang2007preorder, xu2009preorder, chang2009preorder}, using learned reorderings/rewritings based on syntactic patterns \citep{xia2004preorder, li2007preorder, genzel2010preorder, dyer2010preorder, katz2011preorder, lerner2013preorder}, or based on learning-based methods that induce hierarchical features instead of exploiting overt syntactic cues \citep{tromble2009preorder, denero2011preorder, visweswariah2011preorder, neubig2012preorder}. Much of the early work in this setting worked towards integrating the up-and-coming\footnote{At the time.} (phrase-based) statistical machine translation with the longstanding tradition of using syntax in machine translation. With the emergence of purely neural machine translation, recent work has studied how to integrate preorders in an end-to-end fashion using neural methods as well \citep{hoshino2014preorder, de-gispert2015preorder, botha2017preorder, kawara2018preorder}. Especially relevant to the current thesis is the work of \citet {daiber2016preorder}, which  studies the relationship between preorders (and their effectiveness) and the flexibility in word orders in different languages. 
    \item \textit{Postorders} --- Given that there are two sequences involved in machine translation, it is natural to consider postorders as the complement to preorders. However, there is a fundamental asymmetry in that preorders involve changing the input (which can be arbitrarily interacted with) whereas postorders involve changing the output after it has been generated. Therefore postorders require more complex inference procedures and (ideally) require joint training procedures. Given this, postorders have been comparatively under-studied and little evidence has been provided to indicate that there are significant advantages to compensate for these substantial complications when compared to preorders. The one caveat is when developing preorders would be challenging. For example, while generate a preorder for English to Japanese may be viable, generating a preorder for Japanese to English is far more complicated (due to the syntactic patterning of both languages). Therefore, one may use a preorder to improve English to Japanese translation but would struggle to do the same for improving Japanese to English translation. Given these difficulties, a postorder may be attractive in the Japanese to English setting as it is reminiscent of the English to Japanese preorder (and can leverage insights/learned features/parameters from generating an English to Japanese preorder). \citet{sudoh2011postorder} introduced postorders for precisely this reason and \citet{goto2012postorder} extended their method with superior postordering techniques. Further, \citet{mehta2015postorder} introduced an oracle algorithm for generating orders for ten Indian languages but their work received little traction thereafter due to empirical shortcomings. 
\end{itemize}
While reordering to induce alignment has received the most interest in the machine translation, the phenomena is arguably more general. In the extreme, it may be appropriate in \textit{every} task where there are multiple sequences of any type. In particular, \citet{wang-eisner-2018} propose the inspired approach of constructing synthetic languages from high-resource languages (where parsing data is available) whose word order mimics a low-resource language of interest (where parsing data is unavailable/limited) to facilitate cross-lingual transfer in dependency parsing. \citet{rasooli-2019} also consider a similarly reordering method on the source side to improve cross-lingual transfer in dependency parsing.  In particular, it is likely that a similar approach may be of more general value in designing cross-lingual and multi-lingual methods, especially when in the low-resource regime for the language of interest. Very recently, \citet{goyal2020} propose to adapt ideas from work on preorders in machine translation to generate paraphrase. In particular, they repurpose the notion of preorders to construct controllable and flexible preorders based on learned syntactic variations. While most other subareas of NLP have yet to consider word order in dynamic ways, the findings of \citet{galactic-treebank} may prove to be a valuable resource for such study. In this work, the authors introduce the Galactic Treebank, which is a collection of hundreds of synthetic languages that are constructed as hybrids or mutations of real/attested human languages (by intertwining the word order/syntactic patterns of the real natural languages to produce mutants).

\chapter{Word Order in Human Language Processing}\label{chapter:wordorderinhlp}
In this chapter, we examine how word order manifests across languages and within certain contexts. We go on to discuss the relationship between word order and sequential processing, honing in on a memory-based theory known as dependency length minimization.
\section{Ordering Behaviors} \label{sec:orderingbehaviors}
The interpretation of word order is highly language-specific. In particular, the premise that ordering information is meaningful to begin with is itself language-dependent. Languages with \textit{fixed} or \textit{rigid} word orders tend to reliably order constituents in a certain way to convey grammaticality. English is an example of such a language. On the other hand, other languages, such as Russian and Nunggubuyu, may have more flexible word orders and are therefore said to have \textit{free} or \textit{flexible} word orders. Within these languages, some, like Russian, may exhibit multiple word ordering structures but prefer one in most settings; this is known as the \textit{dominant} word order. For other languages, there is no dominant word order, as is the case for Nunggubuyu \citep{nunggubuyu}. In languages with flexible word orders, morphological markings (such as inflection) are frequently employed to convey  information to listeners/comprehenders. In particular, \citet{comrie1981} and \citet{haspelmath1999} have argued that it is precisely these morphological markings that allow flexible word order languages to "compensate" for the information that is not encoder in word order.\footnote{These claims have been disputed by \citet{muller2002free}, but the concern here is the causal relationship between flexible word order and morphological markers. In particular, prior works contest that morphological case is prerequisite to free word order whereas \cite{muller2002free} finds evidence to the contrary. We take no position on this and simply note that morphological markings and flexible word orders often co-occur.} In discussing word order, it is natural to narrow the scope to certain aspects that are of linguistic interest. \\ 

\noindent \textbf{Basic Word Orders.} The ordering of constituents is a standard method for categorizing languages \citep{greenberg1963}. At the coarsest granularity, languages can exhibit different canonical orderings of the \textit{subject} (S), \textit{main verb} (V), and \textit{object} (O) within sentences that feature all three.\footnote{From a linguistic perspective, the terms subject and object are ill-specified. In accordance with standard practice, we will think of the subject as the noun or noun phrase that generally exhibits agent-like properties and the object as the noun or noun phrase that generally exhibits patient-like properties.} In fact, it is standard to refer to this as the language's \textit{basic word order}.  In \autoref{tab:svo}, we depict languages that attest each of the six possible arrangements of S, V, and O as well as typological statistics regarding their relative frequencies. In general, we observe that subject-initial languages constitute an overwhelming fraction of the world's languages and that OSV is the minority ordering by a considerable margin. While such analyses are incomplete \citep{dryer2013wals}\footnote{As many languages exhibit different basic word orders across sentences whereas others. In particular, in a language like German, both SOV and SVO orderings are quite common across sentences. Alternatively, in languages such as Latin and Wampiri, constituents may not be contiguous spans, which may complicate the notion of ordering constituents.}, they offer an immediate illustration that word ordering properties can be of interest typologically \citep{dryer1997, dryer2013typology}. Next, we consider whether these order properties can be used to deduce underlying properties of language as a whole and whether we can formulate theories to explain why these orders arise. \\

\begin{table}[h]
\centering
\begin{tabular}{lccc}
Ordering  & \% Languages & Example Language & Reference  \\ 
\toprule
SOV & 40.99 & Japanese & \citep{japanese} \\

SVO & 35.47 &  Mandarin & \citep{mandarin} \\

VSO & 6.90 & Irish & \citep{irish} \\

VOS & 1.82 & Nias & \citep{nias} \\

OVS & 0.80 & Hixkaryana & \citep{hixkaryana} \\

OSV & 0.29 & Nad\"{e}b & \citep{nadeb} \\

\bottomrule
\end{tabular}
\caption{Basic word orders across the world's languages. Statistics regarding the fraction of the world's languages that primarily use a certain ordering come from \citet{dryer2013wals}. 1376 natural languages were the total number of languages in considering these statistics. References refer to entire works dedicated to studying the corresponding language which rigorously demonstrate the language's dominant word order. The unexplained probability mass corresponds to languages without a dominant word order (e.g.~German) or with discontiguous constituents (e.g~Wampiri).}
\label{tab:svo}
\end{table}

\section{Language Universals}\label{sec:languageuniversals}
Given the set of word ordering effects we have seen so far, it is natural to ask whether certain patterns emerge across language languages. More strongly, one can question whether there are certain \textit{universal} properties which exist (and such hypotheses can be readily tested with experimental and statistical techniques at present). \citet{greenberg1963} initiated this study, with early work towards studying the basic word orders we have seen previously. Greenberg argued that there are three determining factors that specific a \textit{basic typology} over languages:
\begin{enumerate}
    \item A language's basic word order
    \item The prevalence of \textit{prepositions} or \textit{postpositions}. In languages such as Turkish, arguments of a constituent systematically appear before it. In particular, adjectives precede nouns, objects precede words, adverbs precede adjectives, and so forth. For this reason, such a language is labelled \textbf{pre}positional. In contrast, in languages such as Thai, the argument of a constituent systematically appears after it. For this reason, such a language is labelled \textbf{post}positional. Since many languages, such as English display both prepositional behavior (e.g.~adjectives before nouns) and postpositional behavior (e.g.~objects after verbs), Greenberg determined the more prevalent of the two to assign this binary feature to languages.\footnote{Greenberg did not consider circumpositional languages, such as Pashto and Kurdish, where aspects of the argument appear on either side of the constituent. Circumposition is generally observed more frequently at the morphological rather than syntactic level.}
    \item The relative position of adjectives with respect to the nouns they modify. Again, in English, the adjective precedes the noun whereas in Thai, the adjective follows the noun. 
\end{enumerate}
Given these features, there are $24 = 6 \times 2 \times 2$ possible feature triples that a language could display. As the wording of items 2 and 3 suggests, these can be viewed as instances of a broader class of local ordering preferences, we return to this point later in this section.  Greenberg excluded all basic word orders that had objects preceding subjects since he argued that these were never dominant word orders in a language.\footnote{While this claim is false in general, it can be argued to be true for the languages Greenberg studied.} Greenberg then studied 30 natural languages and categorized them into each of these 12 groups. While the statistical validity of his work has been questioned \citep{dryer1988, hawkins1990, dryer1998}, subsequent works (especially in recent times when data is more readily accessible and large-scale corpus analyses can be conducted computationally) have clarified the validity of his theories \citep[e.g.][]{dryer1992, dryer2013wals, hahn2020}. More generally, the enterprise Greenberg initiated of unearthing \textit{language universals} based on consistent patterns across a set of sampled languages has spawn important lines of work in cognitive science and linguistics. \\

\noindent \textbf{Harmonic Word Orders.} Of the language universals that Greenberg put forth, perhaps the most notable have been the harmonic word orders. The term \textit{harmonic} refers to the fact that in some languages, the modifiers of a certain syntactic class (e.g.~nouns) consistently either precede or succeed the class. For example, many languages have both numerals and adjectives precede the noun or both succeed the noun as compared to language where one precedes and the other follows; the latter collection of languages are referred to as \textit{disharmonic}. While there has been significant inquiry towards enumerating languages and the types of (dis)harmonies observed \citep[see][]{hawkins1983}, our interest in harmonic word orders is the cognitive approach towards understanding how they may influence learning. In this sense, harmonic word orders have emerged as a direct line of attack for cognitive inquiry towards connecting word ordering effects, language learning and acquisition, and broader theories of human cognition and language processing. \\

In general, consideration of word order harmonies can be attributed to the reliable and overwhelming statistical evidence. Given this evidence, it is natural to further consider whether a broader cognitive explanation that extends beyond linguistics may be the source for the observed phenomena. One especially relevant line of work has argued that a bias towards word order harmonies can be indicative of general cognitive and/or processing constraints for humans \citep{culbertson2016}. In this sense, word order harmonies contribute to simpler grammars and a proclivity for shorter dependencies that is seen across other domains for human cognition. \citet{culbertson2012} strengthen this position by demonstrating that adult language learners learning artificial/synthetic languages demonstrate strong tendencies towards word order harmonies. \citet{culbertson2015} further extend these results by showing similar behaviors for child language learners while clarifying the distinction with respect to adult language learners regarding the strength and nature of the bias towards harmonic word orders. More recently, \citet{culbertson2017} provide fairly resolute confirmation of this theory and separation of adult and child language learning with regards to harmonic word orders. When both children and adults are tasks with learning languages that are regularly disharmonic, children fail to learn the language correctly and instead innovate/fabricate novel word orders which are harmonic (where the correct harmonic is disharmonic). In contrast, adults are able to replicate the nonharmonic patterns correctly. \\

In our work, while we do not directly appeal to cognitive results for language learning (especially for children), we take this to be motivation that insightful choice of word orders (perhaps in a way that aligns with a learner's inductive bias) can facilitate language acquisition. Conversely, suboptimal choices may complicate language learning substantially and can cause humans (and potentially machines) to resort to alternative strategies that better reconcile the nature of the input with the underlying latent biases. 

\section{Sequential and Incremental Processing} \label{sec:sequentialprocessing}
In the previous section, we catalogued a series of word ordering effects in natural language. Subsequent work has tried to directly explain the word ordering effects and the potential underlying language universals \citep[e.g.][]{hawkins1988}  In many of these cases, the corresponding works in linguistics, psycholinguistics, or cognitive science that studied these phenomena either offered theoretical explanations or empirical evidence. However, a loftier goal for psycholinguistics in particular is to create a broader theory for sequential language processing. In particular, such a theory might explain the word ordering behaviors we have described previously as special cases. \\

 Language is processed incrementally \citep{sag2003}. Consequently, any theory that explains general sequential language processing must grapple with this property of how humans process language. In the study of incremental language processing, the \textit{integration function} is defined to be the function describing the processing difficulty in processing a given word $w_i$ given the preceding context $\langle w_1 \dots w_{i-1} \rangle$ \citep[\citet{ford1982}, c.f.][]{tanenhaus1995, gibson98survey, jurafsky2003}.\footnote{In some works \citep[e.g.][]{venhuizen2019}, additional context beyond the preceding linguistic context, such as the social context or world knowledge, is further modelled. We deliberately neglect any description of such work in our review of past work as we restrict ourselves to language understanding and language modelling that is fully communicated via the preceding linguistic context throughout this thesis.} In both theoretical and empirical inquiry towards understand human incremental language processing, most works make use of some mechanism that allows for controlled variation (e.g.~minimal pair constructions) in the input and analyze the incremental processing difficulty of a human(s) comprehending the input. In empirical work, this analysis is often executed by considering differential effects using a measurement mechanism for human processing (e.g~reading times, reading from the scalp, pupil dilation). \\

The consequence of this work is a canonicalized pair of theories: expectation-based incremental language processing and memory-based incremental language processing. The central tenet of the former is that most processing is done preemptively, since many words can be predicted by their context\footnote{It is this principle that motivates causal language modelling.} and any further difficulty can be attributed to how surprising $w_i$ is given the preceding context. In contrast, the latter theory posits that the integration cost of the current word $w_i$ is proportional to the challenges of integrating it with units that must have been retained in memory. Given the longstanding tradition of studying incremental language processing, joint theories that seek to reconcile the approaches have also been proposed. In particular, given there is strong evidence for both theories (and both often have been showed to be reasonably complementary in their explanatory power), joint theories seek to merge the two, as there is undisputed proof of both predictive processing and memory-driven forgetting in human language processing. 
\subsection{Expectation-based Theories} \label{subsec:expectation-based}
In positing a theory of incremental processing that hinges on prediction, it is necessary to operationalize what is predicted and how processing difficulty emerges from failures in prediction. For this reason, expectation-based theories have largely come to be dominated by surprisal-based theories \citep{hale2001, levy2008}, as the predictions given rise to the processing difficulty inherently. In particular, surprisal is an information-theoretic measure that measures how \textit{surprised} or unlikely a word $w_i$ is given the preceding context $\langle w_1 \dots w_{i-1} \rangle$ as 
\begin{equation}
    \texttt{surp}\left(w_i \mid \langle w_1 \dots w_{i-1} \rangle \right) \triangleq - \log \left(p \left(w_i \mid \langle w_1 \dots w_{i-1} \rangle \right) \right). 
\end{equation}
We will use $\texttt{surp}_\theta$ as notation to denote when the probability distribution $p$ is estimated using a model parameterized by weights $\theta$. From a modelling perspective, many methods have been used to estimate surprisal. In particular, probabilistic context-free grammars \citep{hale2001, levy2008}, classical n-gram language models \citep{smith2013}, recurrent neural network language models \citep{van-schijndel-linzen-2018-neural} and syntactically-enriched recurrent neural networks grammars \citep{rnng, hale2018} have all been used as language models, i.e~choices of $\theta$, to estimate this probability. Crucially, surprisal has been shown to be a reliable predictor of human reading times (robust to six orders of magnitude) by \citet{smith2013}. \\

Surprisal has emerged to be a workhorse of several lines of psycholinguistic inquiry since it provides a natural and strong linking hypothesis between density estimation and human behavior as well as due to its information-theoretic interpretation. In particular, surprisal can be attributed as exactly specifying the change to a representation that is caused by the given word $w_i$ where the representation encodes $\langle w_1 \dots w_{i-1} \rangle$, i.e.~the sequence seen so far. In this sense, surprisal codifies the optimal Bayesian behavior and has come to be part of a broader theory of cognition centered on prediction and predictive coding \citep{friston2009, clark2013}. Further, since it specifies the purely optimal behavior, surprisal retains both the advantages and disadvantages associated with being \textit{representation-agnostic}. We will revisit these in considering motivations for joint theories.  \\

Given these findings, among many others, surprisal theory has strong explanatory power in describing human incremental language processing. As it pertains to this thesis, surprisal has also been recently\footnote{Prior works \citep[e.g][]{cancho2003, cancho2006} also considered information theoretic approaches to language to explain word orders but were considerably less effective than the recent line of work. Further, these works used information theoretical tools but did not necessarily appeal to the expectation-based theories which we consider here.} considered for the purposes of explaining word ordering behaviors. In particular, \citet{hahn2018} demonstrate that surprisal and other information theoretic measures, such as point-wise mutual information, can be used to explain adjective ordering preferences in the sense of Greenberg \citep{greenberg1963}. In particular, they are able to predict adjective orders reliably ($96.2\%$ accuracy) using their cognitive model that is grounded in mutual information and makes use of memory constraints. \citet{futrell-2019-information} also provides similar evidence for word ordering behaviors being explained effectively via the use of information theory. Very recently, \citet{hahn2020} strengthened this position by showcasing that surprisal-based methods can be used to demonstrate that Greenberg's language universals emerge out of efficient optimization within language to facilitate communication.
\subsection{Memory-based Theories} \label{subsec:memory-based}
Under memory-based theories of incremental processing, the processing difficulty of associated with the current word $w_i$ is proportional to the difficulty in/error associated with retrieving units from the context $\langle w_1 \dots w_{i-1} \rangle$. In particular, consider the following four examples \citep[reproduced from][]{lossy-context}:

\ex.\label{ex:localityeffects}
    \a. Bob \depmem{threw out} the trash.
    \b. Bob \depmem{threw} the trash \textbf{out}.
    \c. Bob \depmem{threw out} the old trash that had been sitting in the kitchen for several days.
    \d. Bob \depmem{threw} the old trash that had been sitting in the kitchen for several days \textbf{out}.

Observe that in the first pair of sentences, the sentences are perfectly lexically-matched and both convey identical semantic interpretations. For humans, these sentences have similar processing complexity. However, in the latter pair of sentences, while they are again perfectly lexically-matched and again convey identical semantic interpretations, they have starkly different processing complexities. Humans systematically find sentence (1d) to be more challenging to process than sentence (1c), as has been observed by \citet{lohse2004}. Under memory-based theories, many of which stem from the dependency locality theory of \citet{gibson98locality, gibson2000locality}, this difficulty arises due to the increased length of the dependency between \textbf{threw} and \textbf{out}. In other words, in (1d), a human must retrieve the information regarding \textbf{threw} when processing \textbf{out} and the error in this retrieval or its general difficulty increases as a function of the dependency's length. In particular, to interpret any of these four sentences, it is necessary to process the syntactic dependency linking \textbf{threw} and \textbf{out}; it is insufficient to only process only one lexical item or the other to obtain the correct semantic interpretation \citep{jackendoff2002}. \\

Several hypotheses have been proposed to explain what underlying mechanisms explain the observed increase in dependency as a function of length. Some posit that there is an inherent decay in the quality of the representation in memory over time (consistent with other types of memory representations throughout human cognition) whereas others argue that the degradation is tightly connected with the nature of the intervening material and how it interferes with flawless retention of the context. Regardless, there are numerous effects in linguistics where processing difficulty has been showed to increase with increasing dependency length \citep[e.g.~multiple center-embeddings, prepositional phrase attachment; c.f.][]{lossy-context}. Akin to surprisal theories, there is also evidence that dependency locality and memory-based theories are predictive of human behaviors \citep{grodner2005, bartek2011}. However, some subsequent works have questioned whether dependency locality effects are strong predictors of human behavior beyond the laboratory setting; \citet{demberg2008eye} find no such effects when evaluating using naturalistic reading time data.

\subsection{Joint Theories} \label{subsec:joint-theories}
Given the representation-agnostic nature of expectation-based and surprisal theories of incremental processing and the representation-dependent nature of memory-based theories, joint theories must commit to being either representation-agnostic or representation-dependent, thereby adopting one theory as a basis. Then, these approaches engineer mechanisms by which to integrate the other theory. In general, the motivation for studies towards building joint theories is to capitalize on the observation that expectation-based and memory-based theories of incremental processing have been generally shown to explain complementary phenomena. \\

The Psycholinguistically-Motivated Lexicalized Tree Adjoining Grammar of \citet{demberg2008}, which was further extended in \citet{demberg2009}, \citet{demberg2010}, and \citet{demberg2013}, was one of the first joint approaches. In particular, a parser (under the tree adjoining grammar formalism) is equipped with \textsc{predict} and \textsc{verify} operations. The \textsc{predict} operation is akin to expectation-based predictions of processing difficulty. Dually, the \textsc{verify} operation is is memory-driven as it requires validating that the previously predicted structures are indeed correct (and the cost of this verification scales in the length of the dependencies/varies inversely in the strength of the dependency locality effects). This approach more broadly adopts the perspective of endowing a representation-dependent framework (here specified using the tree adjoining grammar) with predict operations and further constraints. \\

Conversely, \citet{lossy-context} have recently introduced the lossy-context surprisal model which extends the author's previous work on noisy-context surprisal \citep{noisy-context1, noisy-context2, noisy-context3}. In this work, the authors adopt a representation-agnostic perspective grounded in surprisal theory. Based on the observation that pure surprisal theory, which uses information theoretic primitives, cannot account for forgetting effects, the authors suggest making the representation of the context \textit{lossy}. What the authors is a more general concern with information theory, in that information theory in the style of \citet{informationtheory} does not account for models of bounded or imperfect computation. Consequently, if any information can be recovered from the (possibly arbitrarily long) preceding context, information theory will account for this information. Recovering this information without error is likely not viable for humans.\footnote{This fact also likely holds for machines and computational models, which have bounded memory and constrained reasoning capabilities.} \\

\noindent \textbf{Information Locality.} Given the constraints of humans (as have been implicitly shown in the literature on dependency locality), \citet{lossy-context} argue for a theory of information locality, which was first introduced by \citet{futrell-2019-information}. Under such a theory, a memory representation $m_t$ is build at every timestep $t$ and this representation likely imperfectly encodes $\langle w_1 \dots w_t \rangle$.\footnote{If it perfectly encodes the context, pure surprisal theory is recovered.} Consequently, specifying the memory representation (and its forgetting effects) appropriately, via a noise model or other lossy information-encoding mechanism, provides the grounds for addressing the forgetting effects that surprisal theory is ill-equipped to handle. In particular, the authors suggest that operationalizing this by using RNNs with bounded context, as in \citet{alemi2017, hahn2019}, may be an effective approach. We remark that a separate line of inquiry, that directly studies information theory under computational constraints, may be more elegant and sensible. In particular, the theory of $\mathcal{V}$-information put forth by \citet{information-theory-computational-constraints} may prove to be a strong formalism for encoding the information theoretic primitives that ground surprisal as well as the bounded computational resources that induce memory/forgetting effects.   
\section{Dependency Length Minimization} \label{sec:dependencylengthminimization}
Both expectation-based theories such as surprisal theory and memory-based theories such as dependency locality theory have been reasonably effective is explaining human behavior in online language processing. Arguably, the evidence for expectation-based theories is stronger and it is this that prompts the recent development of joint theories that are primarily based on predictive processing \citep{lossy-context}. However, dependency locality theories also have a longstanding tradition of enjoying explanatory power with respect to word ordering effects. In particular, dependency locality theory naturally predicts that humans will produce sentences that employ word orders that minimize dependency length \textit{ceteris paribus}. While similar statement can be made regarding expectation-based theories --- humans use word orders that maximize the predictability of subsequent words --- there is comparatively less evidence.\footnote{However, it should be noted that recent works such as \citet{futrell-2019-information} and \citet{lossy-context} argue for this in instantiating a theory of information locality. In particular, \citet{lossy-context} argue that the word ordering effects suggested by dependency length minimization are merely estimates or approximations of what is truly predicted under information locality by neglecting effects beyond those marked by syntactic dependencies.} \\

A very early predecessor of dependency length minimization is attributed to \citet{behaghel1932}, who stated that "what belongs together mentally is placed close together".\footnote{This sentence is translated from German, as reproduced by \citet{temperley2018}.} Similarly, \citet{greenberg1963} also provided early accounts of dependency length minimization.
More nuance and statistically valid evidence of dependency length minimization has been discovered for many natural languages. \citet{yamashita2001} demonstrated statistically meaningful effects via corpus analysis for Japanese. More recently, \citet{futrell2015} extended these results by showing strong dependency length minimization (well beyond what would be predicted by random word orders), with $p < 0.0001$ for 35 of the 37 languages considered and $p < 0.01$ for the other languages (Telugu and Latin), by making use of the Universal Dependencies Treebank \citep{UD}. Additional evidence has been introduced which suggests that the grammars of natural languages are designed such that word orders which necessitate long-distance dependencies are dispreferred \citep{rijkhoff1990, hawkins1990}. More broadly, dependency length minimization, and therefore the word order preferences it predicts, is a core aspect of a broader argument presented by \citet{hawkins1994, jaeger2011, gibson2019} that natural language emerges as an efficient symbolic system for facilitating human communication from the perspective of both a speaker (language production) and a listener (language comprehension). An excellent multi-faceted survey of the literature on dependency length minimization is provided by \citet{temperley2018}. \\

Given the ubiquitous and diverse grounds for justifying dependency length minimization, computational research has considered the question of devising provably minimal artificial languages to push dependency length minimization to its extreme. While the associated optimization problem of minimizing the cumulative/average dependency length has previously been studied in the algorithms and theory computer science community, with sufficiently general results \citep{minlan3, minlan1.58}, \citet{minlaprojective} introduce an algorithm for finding the word order that provably minimizes dependency subject to projectivity constraints. We discuss this algorithm in \autoref{sec:algorithmsforcombinatorialoptimization}, finding that the algorithm is marginally incorrect, and study its impacts in \autoref{sec:results+analysis}. Further, in the parsing community, biasing parsers to generate short dependencies has proven to be a bona fide heuristic \citep{collins-2003-head, klein2004, eisner2005}. In \citet{smith2006}, the authors note that "$95\%$ of dependency links cover $\leq 4$ words in English, Bulgarian, and Portuguese; $\leq 5$ words in German and Turkish; and $\leq 6$ words in  Mandarin", which provides further evidence to the fact that dependency lengths are minimized and hence are fairly local. 
\chapter{Algorithmic Framing}\label{chapter:algorithmic}
In this chapter, we introduce the algorithmic perspective that we use to formalize our approach towards studying linear order in natural language. 
\section{Notation} \label{sec:notation}
Given a sentence $\bar{s} = \langle w_1, \dots, w_n \rangle$ and it dependency parse $\mathcal{G}_{\bar{s}} = (\mathcal{V}, \mathcal{E}_\ell)$, we will define $\mathcal{E}$ as the unlabelled and undirected underlying edge set of $\mathcal{E}_\ell$. 
\begin{definition} \textit{Linear layout} --- A bijective mapping $\pi: \mathcal{V} \to [n]$.
\end{definition}
\noindent Therefore, a linear layout specifies an ordering on the vertices of $\mathcal{G}_{\bar{s}}$ or, equivalently, a re-ordering (hence a permutation\footnote{This is why we denote linear layouts by $\pi$.}) of the words in $\bar{s}$.
Denote the space of linear layouts on $\bar{s}$ by $S_{n}$\footnote{Formally, $S_n$ denotes the symmetric group on $n$ elements.}. Since the linear order of a sentence innately specifies a linear layout, we define the \textit{identity linear layout}.
\begin{definition} \textit{Identity linear layout} --- A linear layout $\pi_I: \mathcal{V} \to [n]$ specified by: 
\begin{equation}\label{eq:identitylinearlayout}
\pi_I(w_i) = i
\end{equation}
\end{definition}
\begin{definition} \textit{Edge distance/length} --- 
A mapping $d_\pi : \mathcal{E} \to \mathbb{N}$ specified by:
\begin{equation}\label{eq:edgedistance}
d_\pi(w_i, w_j) = |\pi(w_i) - \pi(w_j)|
\end{equation}
\end{definition}
\noindent For example, $d_{\pi_I}(w_i, w_j) = |i - j|$. \\ \\ 
We further introduce the sets $L_i$ and $R_i$ which are the set of vertices to the left (or at) position $i$ or the right of position $i$:
\begin{equation}\label{eq:leftandright}
L_\pi(i) = \{u \in \mathcal{V} : \pi(u) \leq i\} \text{ and } R_\pi(i) = \{v \in \mathcal{V} : \pi(v) > i\}
\end{equation}
\begin{definition} \textit{Edge cut} --- A mapping $\theta_\pi: [n] \to \mathbb{N}$ specified by:
\begin{equation}\label{eq:edgecut}
\theta_\pi(i) = \big{|}\{(u,v) \in \mathcal{E} | u \in L_\pi(i) \wedge v \in R_\pi(i)\} \big{|}
\end{equation}
\end{definition}
\noindent For a more complete introduction on linear layouts, see the survey of \citet{diazlinearlayouts} which details both problems and algorithms involving linear layouts. 
\section{Objectives} \label{sec:objectives}
In studying human language processing, we are inherently constrained to view word order as specified by humans/according to the linear layout $\pi_I$. As we consider alternative orders, we begin by assuming we have a dataset $\mathcal{D}$ of $N$ examples. For every sentence $\bar{s}_i = \langle w_1, \dots, w_n \rangle \in \mathcal{D}$\footnote{For simplicity, we assume every sentence in the dataset is length $n$ in this section.}, there are many possible orderings. Consequently, we define an \textit{ordering rule} $r: \mathcal{D} \to S_n$ as a mapping which specifies a re-ordering for every sentence in $\mathcal{D}$. Given that there are a superexponential number of ordering rules ($n!^N$)\footnote{Recall that we have assumed there are no duplicate words within any sentence.}, it is intractable to explicitly consider every possible ordering rule for such a combinatorially-sized set, even for a single task/dataset.  \\ \\ 
Given that exhaustively considering all orderings is infeasible, we instead cast the problem of selecting a word order for a sentence as a combinatorial optimization problem. Consequently, we define an ordering rule $r_f$, parameterized by an objective function $f$, as follows:
\begin{equation}\label{eq:combinatorialoptimization}
    r_f(\bar{s}_i) = \argmin_{\pi \in S_n}f(\pi, \bar{s}_i)
\end{equation}
for a cost function $f$. In \autoref{sec:futuredirections}, we revisit how we may handle the case when such an optimization is ill-posed/there exist multiple solutions. 
\subsection{Bandwidth} \label{subsec:bandwidth}
In the previous chapter, we presented several accounts that suggest that humans have fundamental limitations on their abilities to process long-distance dependencies. In general, long-distance dependencies can be a substantial complication in maintaining incremental parses of a sentence. As an illustration, consider the example given in Figure~\ref{fig:longdistancedependencyexample} as a particularly pathological case. 
\begin{figure}
  \centering
  \small
  \begin{dependency}[hide label, edge unit distance=.5ex]
    \begin{deptext}[column sep=0.05cm]
      the\& horse\& raced\& past\& the\& barn\& fell \\
    \end{deptext}                                               
  \end{dependency}
    
\caption{A garden path construction with a long-distance dependency linking \textit{horse} and \textit{fell}.}
\label{fig:longdistancedependencyexample}
\end{figure}
Here, the long-distance dependency between \textit{horse} and \textit{fell} may contribute to the confusion in parsing this sentence for many readers on an initial pass. \\

In computational models, we have also seen treatments that restrict the ability to model arbitrarily long dependencies. Most saliently, in models with Markovian assumptions, such as the HMMs described in \autoref{subsec:sequentialmodels}, there is a fundamental constraint that prohibits modelling dependencies of length greater than the Markov order. Similarly, in Transformer models when the stride size is the context window length, dependencies of length greater than the context window length can simply not be modelled.  \\

Given the difficulty of handling long-distance dependencies in both the human language processing and computational language processing settings, we can consider an ordering rule which ensures that the longest dependency in every re-ordered sentence is as short as possible. As such, we define the \textsc{bandwidth} cost function as follows:
\begin{equation}\label{eq:bandwidthcost}
    \textsc{bandwidth}(\pi, \bar{s}) = \max_{(w_i, w_j) \in \mathcal{E}}d_\pi(w_i, w_j)
\end{equation}
Consequently, this defines the ordering rule $r_b$ under the framework given in Equation~\ref{eq:combinatorialoptimization}.
\begin{equation}\label{eq:bandwidthrule}
r_b(\bar{s}) = \argmin_{\pi \in S_n}\textsc{bandwidth}(\pi, \bar{s})
\end{equation}
The term \textit{bandwidth} refers to the fact that the optimization problem given in Equation~\ref{eq:bandwidthrule} is known in the combinatorial optimization literature as the \textsc{bandwidth} problem. The problem was introduced in 1967 by \citet{bandwidthharary} for graphs, though it has been posed previously for matrices in the mid 1950s. The matrix formulation is equivalent, as a graph can be viewed as its adjacency matrix and the bandwidth of a matrix is exactly the bandwidth of a graph as it measures the maximal distance from the main diagonal of any non-zero elements. \\

The bandwidth problem for matrices has a rich history that has been especially prompted by its applications in numerical analysis. Specifically, numerical computations can be improved (generally by reduction of space costs and ) for matrices with low bandwidth in several matrix factorization (e.g.~Cholesky) and matrix multiplication schemes. As a result, bandwidth reduction has been integrated to various numerical analysis software \cite{bandwidthsoftware} and some libraries include algorithms for matrices with small bandwidth \cite{bandwidthfast}. In other contexts, bandwidth reduction has also been combined with other methods for various high-volume information retrieval problems \cite{bandwidthIR}. \\

Beyond its applied value, the problem has been also the subject of deep theoretical inquiry. Extensive surveys have been written on the problem \citep{bandwidthsurvey} as well as the corresponding complexity results \cite{bandwidthcomplexity}. In particular, \citet{bandwidthpapadimitriou} demonstrated the problem was NP-Hard for general graphs. 
\subsection{Minimum Linear Arrangement} \label{subsec:minimumlineararrangement}
In general, using \textsc{bandwidth} as a cost function to induce an ordering rule implies that there are many improvements that are (potentially) missed. For example, a linear layout for a graph with two edges, where the edge lengths are $6$ and $5$ achieves equivalent bandwidth as the the linear layout where the edge lengths are $6$ and $1$. In this sense, the \textsc{bandwidth} objective may be at odds with the realities of language processing as both humans and computers must model every dependency and not just the longest one. \\

Once again, we turn to the prior work on human language processing for inspiration. In particular, we have seen in 
\autoref{sec:dependencylengthminimization} that the literature on dependency length minimization has suggested an alternative processing cost. In particular, the works of \citet{gibson98locality, gibson2000locality} describe a cost function as given in Equation~\ref{eq:minLAcost}. This is the exact cost function used in work demonstrating large-scale evidence of dependency length minimization by \citet{futrell2015}.
\begin{equation}\label{eq:minLAcost}
    \textsc{minLA}(\pi, \bar{s}) = \sum_{(w_i, w_j) \in \mathcal{E}}d_\pi(w_i, w_j)
\end{equation}
As we have seen before, this allows us to define the associated ordering rule $r_m$ under the framework given in Equation~\ref{eq:combinatorialoptimization}.
\begin{equation}\label{eq:minLArule}
r_m(\bar{s}) = \argmin_{\pi \in S_n}\textsc{minLA}(\pi, \bar{s})
\end{equation}
Reminiscent of the \textsc{bandwidth} problem, we refer to this objective as the \textsc{minLA} objective as a shorthand that refers to the \textsc{minimum linear arrangement} problem in the algorithms literature.\footnote{While the objective in Equation~\ref{eq:minLAcost} has also been studied in the psycholinguistics literature under the name of \textit{dependency length}, we choose to use the more general algorithmic jargon. In particular, this helps to disambiguate this objective from others we have seen (such as \textsc{bandwidth}).} Introduced by \citet{minlaharper}, the problem has been referred by various names including \textit{optimal linear ordering} or \textit{edge sum} and is sometimes conflated with its edge-weighted analogue. Harper considered the problem originally in the context of generating effective error-correcting codes \cite{minlaharper, minlaerrorcorrecting}. \\

The problem has arisen in a number of different applications. In particular, in wiring problems for circuit design (e.g.~VLSI layout problems), reductions to the minimum linear arrangement problem are frequent \cite{minlavlsi}. Similarly, the problem has been often used for job scheduling \citep{minlascheduling1, minlascheduling2}. The problem shares important theoretical connections with the \textit{crossing number}, which emerges in aesthetic graph drawing problems \citep{minlagraphdrawing}. As we have seen, the problem and objective are also studied in more abstract settings, removed from the pure combinatorial optimization paradigm, including in the dependency minimization literature and for rudimentary models of neural behavior \citep{minlaneural}. \\ 

Similar to the bandwidth problem, the problem has seen a number of theoretical techniques applied to it. This has led to a number of improvements in complexity results for the problem\footnote{We consider this in \autoref{sec:algorithmsforcombinatorialoptimization}.} for restricted graph families but the general problem is NP-Hard \cite{minlaNPHard}. \citet{minlabenchmarks} have benchmarked the problem in several settings (along with providing approximation heuristics) and \citet{minLAlower} have given general arguments for arriving at lower bounds on the problems. \\

\noindent \textbf{Relating \textsc{bandwidth} and \textsc{minLA}.} The \textsc{bandwidth} and \textsc{minimum linear arrangement} cost functions (Equation~\ref{eq:bandwidthcost} and Equation~\ref{eq:minLAcost}) are related in that both operate over edge lengths with one invoking a max where the other invokes a sum. In this sense, this is highly reminiscent of the relationship shared by $p$ norms for $p = 1$ and $p = \infty$. More generally, we can define a family of ordering rules $r_p$ parameterized by input $p \in \mathbb{N} \cup \infty$ as follows: 
\begin{equation}\label{eq:pnormrule}
    r_p(\pi, \bar{s}) = \argmin_{\pi \in S_n}\bigg{(} \sum_{(w_i, w_j) \in \mathcal{E}}d_\pi(w_i, w_j)^p \bigg{)}^{1/p}
\end{equation}
In particular, setting $p = 1$ recovers the ordering rule $r_m$ for \textsc{minLA} as in Equation~\ref{eq:minLArule} and setting $p = \infty$ recovers the ordering rule $r_b$ for \textsc{bandwidth} as in Equation~\ref{eq:bandwidthrule}. 
\subsection{Cutwidth} \label{subsec:cutwidth}
In introducing the \textsc{bandwidth} and \textsc{minLA} objectives, the motivation was that processing long-distance dependencies was challenging for both humans and machines. With that in mind, the length of the dependencies are not the sole property that may correlate with the complexity of processing (and therefore motivate re-ordering to facilitate  processing). As a complement to the length of dependencies, humans also have limits to their processing capabilities regarding memory capacity. In this sense, having many dependencies simultaneously active/uncompleted may also correlate with cognitive load. This phenomenon has been shown for humans across a variety of fronts, perhaps most famously in the experiments of \citet{7pm2}. Miller demonstrated that humans may have fundamental hard constraints on the number of objects they can simultaneously track in their working short-term memories.\footnote{Canonically, Miller claimed this was $7 \pm 2$.} Given that similar challenges have been found in computational language processing, memory mechanisms and attentional architectures have been proposed to circumvent this issue. Rather than introducing computational expressiveness, we next consider how to devise orders that explicitly minimize quantities related with the number of active dependencies. \\

In order to the track the number of active dependencies, we introduced the notion of the \textit{edge cut} previously, which describes the number of dependencies that begin at or before position $i$ under $\pi$ but have yet to be completed (the other vertex of the edge appears after position $i$ under $\pi$). Consequently, we define the \textsc{cutwidth} cost:
\begin{equation}\label{eq:cutwidthcost}
    \textsc{cutwidth}(\pi, \bar{s}) = \max_{w_i \in \mathcal{V}}\theta_\pi(i)
\end{equation}
As we have seen before, this allows us to define the associated ordering rule $r_c$ under the framework given in Equation~\ref{eq:combinatorialoptimization}.
\begin{equation}\label{eq:cutwidthrule}
r_c(\bar{s}) = \argmin_{\pi \in S_n}\textsc{cutwidth}(\pi, \bar{s})
\end{equation}
Akin to the previous two settings, the \textsc{cutwidth} problem is also a problem in the combinatorial optimization literature pertaining to linear layouts. The problem emerged in the 1970's due to \citet{minlavlsi} and in the late 80's from \citet{cutwidthorigin} as a formalism for circuit layout problems. In particular, the cutwidth of a graph scales in the area needed for representing linear VLSI circuit layouts. \\

As with the previous problems, the \textsc{cutwidth} problem has continued to arise in several other applications. \citet{bandwidthIR} used the \textsc{cutwidth} problem alongside the \textsc{bandwidth} problem for information retrieval, \citet{cutwidthnetwork} studied the problem for designing a \textit{PTAS} for network reliability problems, and \citet{cutwidthautomateddrawing} considered the problem in automated graph drawing. \\

Somewhat unlike the previous two problems, the problem has seen less theoretical interest despite its numerous applications. Nonetheless, the problem was shown to be NP-Hard by \citet{cutwidthnp}, which continues the trend seen for the other combinatorial optimization problems we consider. \\

\noindent \textbf{Linking capacity and dependency-length.} Previously, we introduced the \textsc{minLA} cost function as arguably rectifying an issue with \textsc{bandwidth} cost function. In particular, the \textsc{bandwidth} cost function does not explicitly model the aggregate cost which is ultimately perceived in language processing. Both humans and machine must model and process all dependencies in a sequence in order to fully understand the sentential meaning. A similar inadequacy could be posited regarding the \textsc{cutwidth} optimization problem and objective. Consequently, we define \textsc{sum-cutwidth} as:
\begin{equation}\label{eq:sumcutwidthcost}
    \textsc{sum-cutwidth}(\pi, \bar{s}) = \sum_{w_i \in \mathcal{V}}\theta_\pi(i)
\end{equation}
As we have seen before, this allows us to define the associated ordering rule $r_{m'}$ under the framework given in Equation~\ref{eq:combinatorialoptimization}.
\begin{equation}\label{eq:sumcutwidthrule}
r_{m'}(\bar{s}) = \argmin_{\pi \in S_n}\textsc{sum-cutwidth}(\pi, \bar{s})
\end{equation}
As the naming convention for $r_{m'}$ suggests, we note that we have already encountered $r_{m'}$ and \textsc{sum-cutwidth} previously. \\

\begin{theorem}[Equivalence of average edge cut and average dependency length]\label{thm:minLA-two-interpretations}
$\textsl{\textsc{sum-cutwidth} = \textsc{minLA}}$
\end{theorem}

\begin{proof}
For every edge $(w_i, w_j) \in \mathcal{E}$, the edge contributes its length $d_\pi(w_i, w_j)$ to the \textsc{minLA} cost. On the other hand, edge $(w_i, w_j)$ contributes $1$ to the edge cut $\theta_k$ for $k \in \big[\pi(w_i), \pi(w_j)\big)$.\footnote{\textit{WLOG} assume that $\pi(w_i) < \pi(w_j)$, the notation $[a,b)$ indicates $\{a, a + 1, \dots, b-1\}$ for integral $a,b$.} Therefore, in aggregate, edge $(w_i, w_j)$ contributes exactly $\bigg{|}\big[\pi(w_i), \pi(w_j)\big) \bigg{|} = d_\pi(w_i, w_j)$ to the \textsc{cutwidth} cost. As this holds for every edge, it follows that $\textsc{sum-cutwidth} = \textsc{minLA}$. 
\end{proof}
\begin{corollary}
$r_m = r_{m'}$ up to the uniqueness of the solution of the combinatorial optimization problem. 
\end{corollary}
To the knowledge of the authors, the following argument has not been considered in the literature on modelling language processing with relation to costs pertaining to dependency length/capacity. From this perspective, the result affords an interesting reinterpretation that dependency length minimization is equivalent to minimizing the number of active dependencies. In this sense, it may suggest that related findings (such as those of \citet{7pm2}) may be more pertinent and that the relationship between dependency length and memory capacity may be much more direct than previously believed. 
\section{Algorithms for Combinatorial Optimization} \label{sec:algorithmsforcombinatorialoptimization}
In the previous section (\autoref{sec:objectives}), we introduced objectives and corresponding ordering rules that are motivated by challenges in both computational and human language processing. As a recap of the section, we consider \autoref{fig:objectives-eval}, which depicts a graph and the evaluation of the three cost functions on the graph (given the permutation depicted using vertex labels). Further, in \autoref{fig:different-optimizations}, we observe that solutions to the problem of finding the re-ordering that minimizes each of the three objectives can be different. Note that, in this specific case, the minimum linear arrangement-optimal solution is optimal for the other objectives and the cutwidth solution is optimal for the bandwidth objective but not for the minimum linear arrangement objective.  \\
\begin{figure}
\centering
\begin{tikzpicture}  
  [scale=.9,auto=center,every node/.style={circle,fill=blue!20}] % here, node/.style is the style pre-defined, that will be the default layout of all the nodes. You can also create different forms for different nodes.  
    
    \node (a1) at (0, 0) {1};  
    \node (a2) at (0, 3)  {2}; % These all are the points where we want to locate the vertices. You can create your diagram first on a rough paper or graph paper; then, through the points, you can create the layout. Through the use of paper, it will be effortless for you to draw the diagram on Latex.  
    \node (a3) at (1, 2)  {3};  
    \node (a4) at (2, 1) {4};  
    \node (a5) at (3, 0)  {5};  
    \node (a6) at (2, -1)  {6};  
    \node (a7) at (1, -2)  {7};  
    
    \node (a8) at (0, -3)  {8}; % These all are the points where we want to locate the vertices. You can create your diagram first on a rough paper or graph paper; then, through the points, you can create the layout. Through the use of paper, it will be effortless for you to draw the diagram on Latex.  
    \node (a9) at (-1, -2)  {9};  
    \node (a10) at (-2, -1) {10};  
    \node (a11) at (-3, 0)  {11};  
    \node (a12) at (-2, 1)  {12};  
    \node (a13) at (-1, 2)  {13}; 
  
    \draw (a1) -- (a2); % these are the straight lines from one vertex to another  
    \draw (a1) -- (a3); 
    \draw (a1) -- (a4); % these are the straight lines from one vertex to another  
    \draw (a1) -- (a5); 
    \draw (a1) -- (a6); % these are the straight lines from one vertex to another  
    \draw (a1) -- (a7); 
    \draw (a1) -- (a8); % these are the straight lines from one vertex to another  
    \draw (a1) -- (a9); 
    \draw (a1) -- (a10); % these are the straight lines from one vertex to another  
    \draw (a1) -- (a11); 
    \draw (a1) -- (a12); % these are the straight lines from one vertex to another  
    \draw (a1) -- (a13); 

    \draw (a3) -- (a4); % these are the straight lines from one vertex to another  
    \draw (a4) -- (a5); 
\end{tikzpicture}  
\caption{A graph $\mathcal{G}$ with a linear layout specified by the vertex labels in the figure. Given this linear layout, the bandwidth is $12$ (this is $13 - 1$), the cutwidth is $12$ (this is due to position $1$), and the minimum linear arrangement score is $80$ $\left(\text{this is } \sum_{i=2}^{13} \left(i-1\right) + (4-3) + (5-4) \right)$.} \label{fig:objectives-eval}
\end{figure}
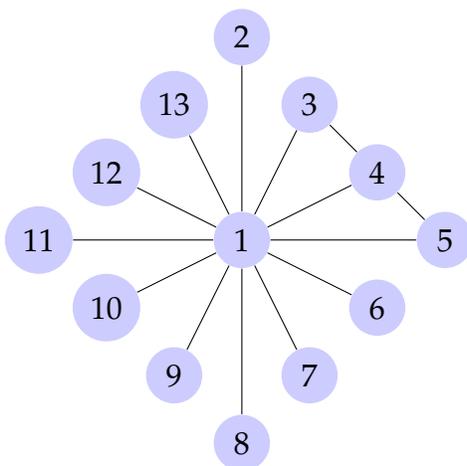

\begin{figure}
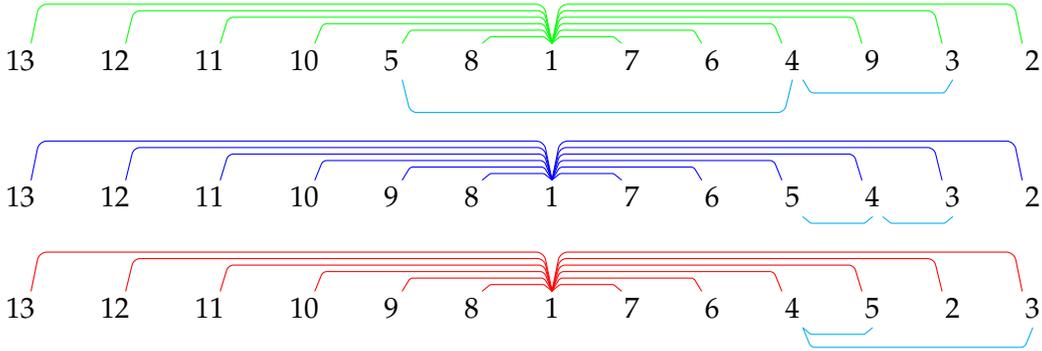

  \centering
  \small
  \begin{dependency}[hide label, edge unit distance=.5ex]
    \begin{deptext}[column sep=0.7cm]
      13 \& 12 \& 11 \& 10 \& 5 \& 8 \& 1 \& 7 \& 6 \& 4 \& 9 \& 3 \& 2  \\
          \end{deptext}
    \depedge[edge style={green}, edge above]{1}{7}{.}
    \depedge[edge style={green}, edge above]{2}{7}{.}
    \depedge[edge style={green}, edge above]{3}{7}{.}
    \depedge[edge style={green}, edge above]{4}{7}{.}
    \depedge[edge style={green}, edge above]{5}{7}{.}
    \depedge[edge style={green}, edge above]{6}{7}{.}
    \depedge[edge style={green}, edge above]{8}{7}{.}
    \depedge[edge style={green}, edge above]{9}{7}{.}
    \depedge[edge style={green}, edge above]{10}{7}{.}
    \depedge[edge style={green}, edge above]{11}{7}{.}
    \depedge[edge style={green}, edge above]{12}{7}{.}
    \depedge[edge style={green}, edge above]{13}{7}{.}
    \depedge[edge style={cyan}, edge below]{5}{10}{.}
    \depedge[edge style={cyan}, edge below]{10}{12}{.}
    
  \end{dependency}
  
    \begin{dependency}[hide label, edge unit distance=.5ex]
    \begin{deptext}[column sep=0.7cm]
      13 \& 12 \& 11 \& 10 \& 9 \& 8 \& 1 \& 7 \& 6 \& 5 \& 4 \& 3 \& 2  \\
          \end{deptext}
    \depedge[edge style={blue}, edge above]{1}{7}{.}
    \depedge[edge style={blue}, edge above]{2}{7}{.}
    \depedge[edge style={blue}, edge above]{3}{7}{.}
    \depedge[edge style={blue}, edge above]{4}{7}{.}
    \depedge[edge style={blue}, edge above]{5}{7}{.}
    \depedge[edge style={blue}, edge above]{6}{7}{.}
    \depedge[edge style={blue}, edge above]{8}{7}{.}
    \depedge[edge style={blue}, edge above]{9}{7}{.}
    \depedge[edge style={blue}, edge above]{10}{7}{.}
    \depedge[edge style={blue}, edge above]{11}{7}{.}
    \depedge[edge style={blue}, edge above]{12}{7}{.}
    \depedge[edge style={blue}, edge above]{13}{7}{.}
    \depedge[edge style={cyan}, edge below]{10}{11}{.}
    \depedge[edge style={cyan}, edge below]{11}{12}{.}
    
  \end{dependency}
  
      \begin{dependency}[hide label, edge unit distance=.5ex]
    \begin{deptext}[column sep=0.7cm]
      13 \& 12 \& 11 \& 10 \& 9 \& 8 \& 1 \& 7 \& 6 \& 4 \& 5 \& 2 \& 3  \\
          \end{deptext}
    \depedge[edge style={red}, edge above]{1}{7}{.}
    \depedge[edge style={red}, edge above]{2}{7}{.}
    \depedge[edge style={red}, edge above]{3}{7}{.}
    \depedge[edge style={red}, edge above]{4}{7}{.}
    \depedge[edge style={red}, edge above]{5}{7}{.}
    \depedge[edge style={red}, edge above]{6}{7}{.}
    \depedge[edge style={red}, edge above]{8}{7}{.}
    \depedge[edge style={red}, edge above]{9}{7}{.}
    \depedge[edge style={red}, edge above]{10}{7}{.}
    \depedge[edge style={red}, edge above]{11}{7}{.}
    \depedge[edge style={red}, edge above]{12}{7}{.}
    \depedge[edge style={red}, edge above]{13}{7}{.}
    \depedge[edge style={cyan}, edge below]{10}{11}{.}
    \depedge[edge style={cyan}, edge below]{10}{13}{.}
    
  \end{dependency}

\caption{Solutions for optimizing each of the three objectives for the graph given in \autoref{fig:objectives-eval}. The linear layout is conveyed via the linear ordering and the numbers refer to the original vertices in the graph (as shown in \autoref{fig:objectives-eval}). The top/{\color{green}green} graph is bandwidth-optimal (bandwidth of 6), the middle/{\color{blue}blue} graph is minimum linear arrangement-optimal (minimum linear arrangement score of 44), the bottom/{\color{red}red} graph cutwidth-optimal (cutwidth of 6). The {\color{cyan}cyan} edges drawn below the linear sequence convey the difference in the optimal solutions.}
\label{fig:different-optimizations}
\end{figure}

In order to make use of these ordering rules for natural language processing applications, it is necessary to tractable solve each of the associated optimization problems. Recall that each of these problems is NP-Hard for general graphs. In spite of this roadblock, also recall that we are considering applying this optimization to sentences equipped with dependency parses, where dependency parses are graphs that are guaranteed/constrained to be trees. \\ 

\noindent \textbf{Bandwidth.} 
Unfortunately, the \textsc{bandwidth} problem remains NP-Hard for trees as well \citep{bandwidthcomplexity}. In fact, the problem is well-known for remaining NP-Hard for a number of graph relaxations \citep{bandwidthnphardcaterpillar, bandwidthnphardrelax} including fairly simple graphs like caterpillar with hair-length at most 3 \citep{bandwidthnphardcaterpillar}. Given this, one natural approach to find tractable algorithms is to consider provable approximations. However, approximations to a factor of $1.5$ do not even exist for both general graphs and trees \citep{bandwidthnoapprox}.\footnote{This further implies that the \textsc{bandwidth} problem does not admit a \textit{PTAS}.} Regardless, approximation guarantees are generally unhelpful as we are considering sentences-scale graphs and therefore \textit{small} graphs, where the approximation factor may be quite significant. Instead, in this thesis, we consider heuristics for the \textsc{bandwidth} problem. In particular, we make use of the frequently-employed heuristic of \citet{bandwidthheuristic}. We discuss this algorithm below and defer the implementation details to a subsequent chapter. \\

The Cuthill-McKee algorithm begins by starting at the vertex and conducting a breadth-first search (BFS) from that vertex. The key to the algorithm's empirical effectiveness is that vertices are visited based on their degree (hence the starting vertex is the vertex with lowest degree). Instead of using the standard algorithm, we use the Reverse Cuthill-McKee algorithm \citep{reverse-cuthill-mckee}, which merely executes the algorithm with reversed index numbers. In theory, this modification has no benefits for general graphs but empirically, it seems this modification turns out to be reliably beneficial \citep{direct-methods-sparse-linear-systems, reverse-cuthill-mckee-practice}.    \\

\noindent \textbf{Minimum Linear Arrangement.} 
Unlike the \textsc{bandwidth} problem, the \textsc{minLA} problem has poly-time solutions for trees. In particular, in a series of results, the run-time for the tree setting was improved from $\mathcal{O}(n^3)$ \citep{minlan3} to $\mathcal{O}(n^{2.2})$ \citep{minlan2.2} to $\mathcal{O}(n^{1.58})$ \citep{minlan1.58}. While there has been progress on developing lower bounds \citep{minLAlower}, matching bounds have been yet to be achieved. In this work, we elect not to use the algorithm of \citet{minlan1.58} and instead introduce an additional constraint and a corresponding algorithm in a subsequent section (\autoref{subsec:projectivityconstraints}). \\

\noindent \textbf{Cutwidth.} 
Similar to \textsc{minLA}, \textsc{cutwidth} also permits poly-time solutions for trees. While the problem remained open as to whether this was possible for a number of years, \citet{cutwidth-nlogn/linear} gave a $\mathcal{O}(n\log(n))$ algorithm. Further akin to \textsc{minLA}, we forego this general algorithm (for trees) for one that involves projectivity constraints (\autoref{subsec:projectivityconstraints}). 

\subsection{Projectivity Constraints} \label{subsec:projectivityconstraints}
Recall that poly-time algorithms exist for both the \textsc{minLA} and \textsc{cutwidth} problems. In both cases, the works introducing the corresponding algorithms we discussed previously develop significant algorithmic machinery to arrive at the final algorithms. We will next see that a linguistically-motivated constraint in \textit{projectivity} yields faster\footnote{From a practical perspective, the algorithms given in \autoref{sec:algorithmsforcombinatorialoptimization} are likely sufficiently fast for almost any language data. In particular, sentences are generally short from a complexity perspective and hence algorithms with asymptotic complexity of $\mathcal{O}(n^{1.58})$ and $\mathcal{O}(n\log n)$ are unlikely to be of concerning cost, especially since the methods we study in \autoref{chapter:optimallinearordersforNLP} are one-time costs.} and simpler algorithms. Recall that projectivity refers to the property of a dependency parse that when the nodes are ordered on a line and the edges are drawn above the line, the parse has no intersecting edges. If we constraint the order $\pi$ outputted by either $r_m$ or $r_c$ to be projective, linear time algorithms are known for \textsc{minLA} \citep{minlaprojective} and \textsc{cutwidth} \citep{cutwidth-nlogn/linear}.\footnote{In some algorithmic contexts, the problem of returning a linear layout that is constrained to be projective is known as the \textsc{planar} version of a problem, e.g.~\textsc{planar-cutwidth}.} We next discuss the algorithms of \citet{minlaprojective} and \citet{cutwidth-nlogn/linear}, showing how they both can be generalized using the framework of \textit{disjoint strategies} \citep{cutwidth-nlogn/linear}. \\
\begin{figure}
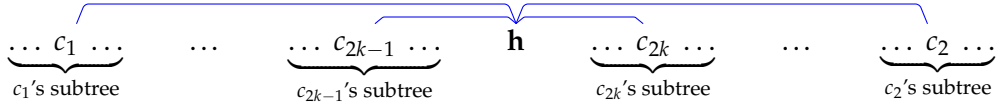

  \centering
  \small
  \begin{dependency}[hide label, edge unit distance=.5ex]
    \begin{deptext}[column sep=0.7cm]
    $\underbrace{\dots \ c_1 \ \dots}_{c_1\text{'s subtree}}$ \& \dots \& $\underbrace{\dots \ c_{2k-1} \ \dots}_{c_{2k-1}\text{'s subtree}}$ \& \textbf{h} \& $\underbrace{\dots \ c_{2k} \ \dots}_{c_{2k}\text{'s subtree}}$ \& \dots \& $\underbrace{\dots \ c_{2} \ \dots}_{c_{2}\text{'s subtree}}$  \\ \\
          \end{deptext}
    \depedge[edge style={blue}, edge above]{1}{4}{.}
    \depedge[edge style={blue}, edge above]{3}{4}{.}
    \depedge[edge style={blue}, edge above]{5}{4}{.}
    \depedge[edge style={blue}, edge above]{7}{4}{.}
  \end{dependency}
\caption{Illustration of the disjoint strategy. The root $h$ is denoted in \textbf{bold} and it has $2k$ children denoted by $c_1, \dots, c_{2k}$. Its children and their subtrees are organized on either side. The order within each child subtree is specified by a linear layouts that has previously been computed in the dynamic program. The order of the children and their subtrees alternates and moving from outside to inside based on their score according to some scoring function. Hence, the subtree rooted at child $c_1$ receives the highest score and the subtree roots at child $c_{2k}$ receives the lowest score. If the root $h$ had $2k+1$ (an odd number) of children, the strategy is followed for the first $2k$ and we discuss the placement of the last child subsequently.}
\label{fig:disjointstrategy}
\end{figure}

\begin{algorithm}
\DontPrintSemicolon
\SetAlgoLined
Input: A tree rooted at $h$ with children $c_1, \dots, c_{2k}$. \\
Input: Optimal linear layouts $\pi_1, \dots, \pi_{2k}$ previously computed in the dynamic program. $\pi_i$ is the optimal linear layout for the tree rooted at $c_i$. \\
$\pi_h \gets \{h : 1\} $\;
$\text{ranked-children} \gets \texttt{sort}\big(\left[1, 2, \dots, {2k}\right], \lambda x. \texttt{score}\left(c_x\right)\big)$\;
$\pi \gets \left(\bigoplus_{i=1}^k \pi_{\text{ranked-children}[2i-1]}\right) \oplus \pi_h \oplus \left(\bigoplus_{i=0}^{k-1} \pi_{\text{ranked-children}[2(k-i)]}\right)$\;
\Return $\pi$
\label{alg:disjointstrategy}
\caption{\texttt{Disjoint Strategy}}
\label{alg:disjointstrategy}
\end{algorithm}

\noindent \textbf{Dynamic Programming for tree-based optimization.} 
Since we are considering both the \textsc{cutwidth} and \text{minLA} problems in the setting of trees, dynamic programming approaches are well-motivated. In particular, we will consider how optimal linear layouts for each of the children at a given node can be integrated to yield the optimal linear layout at the given node. As we will show, both algorithms we consider can be thought of as specific instantiations of the abstract \textit{disjoint strategy} introduced by \citet{cutwidth-nlogn/linear}.\footnote{\textbf{Remark:} \citet{minlaprojective} develop the same general framework in their own work. Since both algorithms share a similar template, we prefer the standardized template to their \textit{ad hoc} description.} In \autoref{fig:disjointstrategy}, we depict what the disjoint strategy looks like and in Algorithm~\ref{alg:disjointstrategy}, we provide the template for the disjoint strategy algorithm. We denote linear layouts programmatically as dictionaries, hence $\pi_h$ is the function $\pi_h: \{h\} \to \{1\}$ given by $\pi_h(h) = 1$. We define the $\oplus$ operator over linear layouts below.
\begin{definition} $\oplus$ --- A binary operator for arbitrary parameters $n, m$ of type $\oplus: S_n \times S_m \to S_{n+m}$ specified by: 
\begin{equation}\label{eq:oplus}
\oplus(\pi_x, \pi_y)(w_i) = \begin{cases} 
      \pi_x(w_i) & w_i \in Dom(\pi_x) \\
      \pi_y(w_i) + n & w_i \in Dom(\pi_y) 
   \end{cases}
\end{equation}
$\bigoplus$ is given by the repeated application of $\oplus$ (analogous to the relationship between $+$ and $\sum$ or $\cup$ and $\bigcup$). 
\end{definition}

\noindent \textbf{Minimum Linear Arrangement.} The function \texttt{score} in Algorithm~\ref{alg:disjointstrategy} is defined such that $\texttt{score}(c_i)$ is the size of the subtree rooted at $c_i$. \\

\noindent \textbf{Cutwidth.} The function \texttt{score} in Algorithm~\ref{alg:disjointstrategy} is defined such that $\texttt{score}(c_i)$ is the modified cutwidth of the subtree rooted at $c_i$ under $\pi_i$. The modified cutwidth of a tree under $\pi$ is the cutwidth of the tree under $\pi$ plus a bit indicating whether that are positions on either side of the root of the tree at which the cutwidth (the maximum edge cut) is obtained. Hence, the modified cutwidth is the cutwidth if all positions at which the edge cut is maximized occur strictly on the left of the root XOR if all positions at which the edge cut is maximized occur strictly on the right of the root. Otherwise, the modified cutwidth is the cutwidth plus one.  \\

\noindent \textbf{Proofs of correctness.} Intuitively, both proofs of correctness hinge on the fact that the disjoint strategy involves placing `larger' subtrees in an alternating outside-to-inside fashion around the current node/root. By selecting the correct measure of `large', the `adverse' effects of the large subtrees affect as few other subtrees as possible (since they are outside the link from these subtrees to the root, which is the only way in which the costs interact across subtrees/with the root). For readers interested in the fully formal proof of correctness, we direct them to the corresponding works: \citet{cutwidth-nlogn/linear} and \citet{minlaprojective}.

\begin{figure}
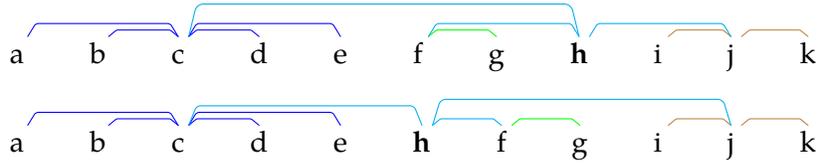

  \centering
  \small
  \begin{dependency}[hide label, edge unit distance=.5ex]
    \begin{deptext}[column sep=0.7cm]
    a \& b \& c \& d \& e \& f \& g \& \textbf{h} \& i \& j \& k  \\
          \end{deptext}
    \depedge[edge style={blue}, edge above]{1}{3}{.}
    \depedge[edge style={blue}, edge above]{2}{3}{.}
    \depedge[edge style={blue}, edge above]{3}{4}{.}
    \depedge[edge style={blue}, edge above]{3}{5}{.}
    \depedge[edge style={cyan}, edge above]{3}{8}{.}
    \depedge[edge style={green}, edge above]{6}{7}{.}
    \depedge[edge style={cyan}, edge above]{6}{8}{.}
    \depedge[edge style={brown}, edge above]{9}{10}{.}
    \depedge[edge style={brown}, edge above]{10}{11}{.}
    \depedge[edge style={cyan}, edge above]{8}{10}{.}
    
  \end{dependency}
    \begin{dependency}[hide label, edge unit distance=.5ex]
    \begin{deptext}[column sep=0.7cm]
    a \& b \& c \& d \& e \& \textbf{h} \& f \& g \& i \& j \& k  \\
          \end{deptext}
    \depedge[edge style={blue}, edge above]{1}{3}{.}
    \depedge[edge style={blue}, edge above]{2}{3}{.}
    \depedge[edge style={blue}, edge above]{3}{4}{.}
    \depedge[edge style={blue}, edge above]{3}{5}{.}
    \depedge[edge style={cyan}, edge above]{3}{6}{.}
    \depedge[edge style={green}, edge above]{7}{8}{.}
    \depedge[edge style={cyan}, edge above]{6}{7}{.}
    \depedge[edge style={brown}, edge above]{9}{10}{.}
    \depedge[edge style={brown}, edge above]{10}{11}{.}
    \depedge[edge style={cyan}, edge above]{6}{10}{.}
  \end{dependency}

\caption{Linear layouts exemplifying the difference between the solutions produced by the \citet{minlaprojective} algorithm (top) and our algorithm (bottom). The root $h$ is denoted in \textbf{bold}. In both algorithms, the linear layouts for the children with the largest subtrees --- the {\color{blue}{blue}} subtree rooted at $c$ and the {\color{brown}{brown}} subtree rooted at $j$ --- are placed on opposite sides. The difference is the placement of the {\color{green}{green}} subtree rooted at child $f$. The arcs whose edge lengths change across the two layouts are those in {\color{cyan}{cyan}}, notably $(c,h), (f,h),$ and $(j,h)$. However, the sum of the edge lengths for $(c,h)$ and $(j,h)$ is constant across the linear layouts. Hence, the difference in minimum linear arrangement scores between the linear layouts is solely dictated by the length of $(f,h)$, which is shorter in our algorithm's layout (bottom layout).}
\label{fig:minLAcorrection}
\end{figure}

\noindent \textbf{Correcting the algorithm of \citet{minlaprojective}.} 
We note that one small correction is made in our algorithm that was not originally addressed by \citet{minlaprojective}. In particular, consider the case when the given node has an odd number of children (i.e.~$2k+1$ children for some non-negative integer $k$). In this case, the $2k$ largest children and their associated subtrees are alternated as dictated by the \textit{disjoint strategy}. \citet{minlaprojective} claim that the placement of the final child does not matter, however this is incorrect. We show this in \autoref{fig:minLAcorrection}. That said, it is fairly simple to correct this error. The final child's subtree is \textit{left-leaning} if more of the child's subtree is oriented to its left than right (according to the linear layout already computed), \textit{right-leaning} if more of the child's subtree oriented to its right than left, and \textit{centered} otherwise. If the child's subtree is leaning in either direction, it should be placed on that side of the root (closest to the root relative to the root's other children, i.e.~in accordance with the \textit{disjoint strategy}). Otherwise, the side it is placed on does not matter.  \\

The proof of correctness that this globally improves the outputted linear layout is fairly straightforward. In particular, since there are $k$ children of the root on either side, the objective will be incremented by $k \ *$ the size of the $2k+1$ child's subtree in irrespective of this decision (where to place child $2k+1$ and its subtree). All arcs above the root and within any of the root's children's subtrees will not change length. Hence the only arc of interest is the one connecting child $2k+1$ and the root. In this case, clearly placing the root on the side opposite of the way the child's subtree is leaning will yield a smaller such arc (as depicted in \autoref{fig:minLAcorrection}). \\

As a concluding remark, we note that the error we presented with the algorithm of \citet{minlaprojective} is exceptionally pervasive. In particular, for every child with an odd number of children in the tree, the algorithm they present may have contributed to a misplacement with probability $\frac{1}{2}$. When evaluated over data we describe in \autoref{sec:data}, we find over $95\%$ of the sentences contain such an instance. Additionally, the runtime analysis of \citet{minlaprojective} indicates that their algorithm is $\mathcal{O}(n)$. Given the sorting required, this is not naively true but can be rectified using appropriate data structures and bucket sort as suggested by \citet{cutwidth-nlogn/linear}. 

\section{Heuristics for Mixed-Objective Optimization}\label{sec:heuristics}
In the previous section, we considered an algorithm for each of the \textsc{bandwidth}, \textsc{minLA}, and \textsc{cutwidth} problems. In all cases, the algorithm had the goal of a producing that a linear layout that minimized the corresponding objective to the extent possible.\footnote{Perhaps with additional constraints such as projectivity.} However, in our setting, we are considering re-ordering the words of a sentence for modelling sentences. From this perspective, there is a clear risk that re-ordering the words to optimize some objective regarding dependencies may obscure other types of order-dependent information in the original sentence. We next consider a template that specifies a simple heuristic for each of these optimization problems. In particular, the heuristic involves a parameter $T$ which bears some relationship with the notion of trading off the extent to which the original sentence's word order is distorted with the extent to which the combinatorial objective is minimized. 
\subsection{Transposition Monte Carlo}\label{subsec:transpositionalgorithm}
\begin{algorithm}
\DontPrintSemicolon
\SetAlgoLined
Input: A sentence $\bar{s}$ and its dependency parse $\mathcal{G}_{\bar{s}} = (\mathcal{V}, \mathcal{E})$.\;
Initialize $\pi = \pi_I$\;
Initialize $c = \texttt{OBJ}(\pi, \bar{s})$\; 
\For{$t \gets 1, \dots, T$}{
$w_i, w_j \sim \mathcal{U}_{\mathcal{V}}$ \;
$\pi_{\text{temp}} \gets \pi$\; 
$\pi_{\text{temp}}(w_i), \pi_{\text{temp}}(w_j) \gets \pi_{\text{temp}}(w_j), \pi_{\text{temp}}(w_i)$\; 
$c_\text{temp} \gets \texttt{OBJ}(\pi_{\text{temp}}, \bar{s})$\; 
\uIf{$c > c_\text{temp}$}{
$\pi \gets \pi_{\text{temp}}$\;
$c \gets c_\text{temp}$\;
}
}
\Return $\pi$
\label{alg:transposition-template}
\caption{\texttt{Transposition Monte Carlo}}
\label{alg:transposition-template}
\end{algorithm}
In Algorithm~\ref{alg:transposition-template}, we present the template we consider for our three heuristic algorithms. Intuitively, the algorithm is a Monte Carlo algorithm that at each time step randomly samples a transposition\footnote{The notation permits the $w_i = w_j$ but there is no benefit to this, so the notation should be taken as sampling $w_i$ and then sampling $w_j$ without replacement from the resulting distribution.} and considers altering the current permutation $\pi$ according to this transposition. For this reason, we call the algorithm the \texttt{Transposition Monte Carlo} algorithm. \\

\noindent We refer to this algorithm as a heuristic since, like any Monte Carlo algorithm, there is no provable guarantee that the algorithm produces the optimal solution. In particular, we note that the algorithm greedily decides whether to update in accordance with a candidate transposition and hence makes \textit{locally optimal} decisions. Consequently, there is no guarantee that this procedure will lead to the \textit{globally optimal} linear layout. However, unlike the algorithms in the section on algorithms for producing linear layouts under projectivity constraints (\autoref{subsec:projectivityconstraints}), the permutation produced by this algorithm need not be projective. We will revisit this in empirically considering the quality of the optimization on natural language data (\autoref{sec:data}). \\

\noindent The \texttt{Transposition Monte Carlo} algorithm is parameterized by two quantities: the objective/cost function being minimized \texttt{OBJ} and the stopping criterion/threshold $T$. 
\begin{enumerate}
    \item \texttt{OBJ} --- 
    In this work, we specify cost functions \texttt{OBJ} in accordance with those which we have seen previously --- the \textsc{bandwidth} cost in \autoref{eq:bandwidthcost}, the \textsc{minLA} cost in \autoref{eq:minLAcost}, and the \textsc{cutwidth} cost in \autoref{eq:cutwidthcost}. We will refer to the associated permutations as $\tilde\pi_{{b}}$, $\tilde\pi_{{m}}$, and $\tilde\pi_{{c}}$ respectively and will similarly denote the induced ordering rules as $\tilde r_{{b}}$, $\tilde r_{{m}}$, and $\tilde r_{{c}}$.
    \item $T$ --- In this work, we fix $T = 1000$ which approximately corresponds to \texttt{Transposition Monte Carlo} taking the same amount of wall-clock time as it takes to run any of the algorithms in \autoref{sec:algorithmsforcombinatorialoptimization} on the same data. However, varying $T$ allows for the possibility of flexibility controlling the extent to which the returned linear layout is distorted. For $T = 0$, we recover the ordering rule $r_I$ which returns $\pi_I$ for a given input sentence. For $T = \infty$, we are not guaranteed to find a global optima to the author's knowledge (due to the local greediness of the algorithm), but we are guaranteed to find a local optima (in the sense that no transposition from the solution yields a better solution). Treating $T$ as a task-dependent hyperparameter in NLP applications and optimizing for it (perhaps on validation data) may encode the extent to which dependency locality is important for the given task.  
\end{enumerate}
% \chapter{Optimal Linear Orders and Natural Language}\label{chapter:optimallinearordersandNL}
\chapter{Optimal Linear Orders for Natural Language Processing}\label{chapter:optimallinearordersforNLP}
In this chapter, we describe how to integrate the novel word orders we have considered with models in natural language processing. We specifically introduce the \texttt{pretrain-permute-finetune} framework which fully operationalizes this. We then conduct an empirical evaluation of our method using the word orders we have introduced previously. 
\section{Motivation}\label{sec:motivation}
In \autoref{sec:related-work}, we discuss several approaches to addressing order, and word order specifically, in computational models of language. The dominant recent paradigm has been to introduce computational expressiveness (e.g.~LSTMs in place of simple RNNs, attention, syntax-augmented neural methods \citep{recursivenn, recursivenn+logic, rnng, rnn+syntax?, unsupervisedrnng}) as a mechanism for handling the difficulties of long-distance dependencies and word order more generally. Alternatively, position-aware Transformers have been considered but their widespread effectiveness hinges on large amounts of data with sufficient hardware resources to exploit the increased ability for parallelism. These methods also tend to be significantly more challenge to optimize in practice \citep{trainingtransformers}. Ultimately, it remains unclear whether it is necessary to rely on additional model complexity or whether restructuring the problem (and the input specifically) may be more appropriate. To this end, we consider the orders which we have introduced in \autoref{chapter:algorithmic} as a mechanism for codifying new views to the problem of sequential modelling language. Next, we explore the potential newfound advantages (and costs/limitations) of using these novel word orders in computational models of language. 
\section{Methods}\label{sec:methods}
We began by considering training models where, for every sentence in the training data, we instead use its re-ordered analogue. In initial experiments, we found the results to be surprisingly promising. However, we found that this paradigm may not be particularly interesting given there are very few, if any, tasks in NLP that are currently conducted using representations built upon strictly task-specific data. Instead, almost all NLP models leverage pretrained (trained on large amounts of unlabelled data in a downstream task-agnostic fashion) representations to serve as potent initializations \citep{ruderthesis}. The past few years have seen the introduction of pretrained \textit{contextualized representations}, which are functions $c: \bigcup_{i=1}^\infty\mathcal{V}^i \to \bigcup_{i=1}^\infty\left(\mathbb{R}^d\right)^i$, that map word sequences (e.g.~sentences) to vector sequences of equal length. In particular, the resulting vectors encode the contextual meaning of the input words. Initial pretrained contextualized representations include: CoVe \citep{cove}, ELMo \citep{elmo}, and GPT \citep{gpt}; the current best\footnote{It is not sufficiently specified to rank pretrained representations in a task-agnostic fashion given that they may have disparate and inconsistent performance across different downstream tasks/datasets. In labelling some of these representations as the `best', we refer to the \texttt{GLUE} \citep{glue} and \texttt{SuperGLUE} \citep{superglue} benchmarks for natural language understanding and similar benchmarks for natural language generation \citep[c.f][]{t5}.} pretrained contextualized representations include: BERT \citep{bert}, GPT-2 \citep{gpt2}, XLNet \citep{xlnet}, RoBERTa \citep{roberta}. SpanBERT \citep{spanbert}, ELECTRA \citep{electra}, ALBERT \citep{albert} and T5 \citep{t5}. \\

\noindent Given the widespread use of transfer learning and pretraining methods in NLP \citep{ruderthesis}, we begin by describe the \texttt{pretrain-and-finetune} framework that has featured prominently across NLP. In particular, the mapping from an input sentence $\bar{s}$ to the predicted output $\hat{y}$ is given by the following process:
\begin{enumerate}
    \item Tokenize the input sentence $\bar{s}$ into words $\langle w_1 \dots w_{n} \rangle$.
    \item Embed the sentence $\langle w_1 \dots w_{n} \rangle$ using the pretrained encoder/contextualized model $c$ to yield vectors $\vec{x}_1, \dots, \vec{x}_{n}$ where $\forall i \in \left[n\right], \vec{x}_i \in \mathbb{R}^d$.
    \item Pass $\vec{x}_1, \dots, \vec{x}_{n}$ into a randomly initialized component $F$ that is trained on the task of interest that outputs the prediction $\hat{y}$. 
\end{enumerate}
In order to modify this process to integrate our novel word orders, one straightforward approach would be simply feeding the permuted sequence of inputs into the pretrained encoder. Perhaps unsurprisingly, we found this to perform quite poorly in initial experiments. After all, the pretrained encoder never observed such permuted constructions during training and these inputs are effectively out-of-distribution for the model. Another natural approach that is more reasonable would be to re-pretrain the encoder using the same training data but with all sentences permuted according to the order being studied. Unfortunately, this is well beyond our computational constraints and likely is not practical for almost all research groups given the tremendous time, cost, and unique GPU and TPU resources required to pretrain modern models. Even for such entities who can bear these expenses, this is unlikely to be feasible if multiple orders are to be experimented with and there are substantial ethical considerations given the sizeable environmental impact \citep{environment}.  \\

Given these roadblocks, we propose permuting the vectors $\vec{x}_1, \dots, \vec{x}_{n}$ in  between steps 2 and 3. In doing so, we seamlessly integrate our permuted orders in a model-agnostic and task-agnostic fashion while preserving the benefits of pretraining. Further, since the permutation for a given example can be pre-computed (and is a one-time cost), the additional runtime cost of our method during both training and inference is simply the cost of computing the optimal orders over the dataset.\footnote{As a reference, for the 5 datasets we study (and over 40000 examples), optimization across all six orders takes less than 3 hours on a single CPU. Several aspects of this are embarrassingly parallel and others can be more cleverly designed using vector/matrix operations in place of dynamic programming and for loops. This suggests that highly parallel GPU implementations can render this runtime cost to be near-zero. Regardless, it is already dramatically dwarfed by the costs of training.} We name our framework \texttt{pretrain-permute-finetune} and we explicitly state the procedure below. 
\begin{enumerate}
    \item Tokenize the input sentence $\bar{s}$ into words $\langle w_1 \dots w_{n} \rangle$.
    \item Embed the sentence $\langle w_1 \dots w_{n} \rangle$ using the pretrained encoder/contextualized model $c$ to yield vectors $\vec{x}_1, \dots, \vec{x}_{n}$ where $\forall i \in \left[n\right], \vec{x}_i \in \mathbb{R}^d$.
    \item Permute the vectors $\vec{x}_1, \dots, \vec{x}_{n}$ according to linear layout $\pi$. In other words, $\forall i \in \left[n\right], \vec{z}_{\pi\left(w_i\right)} \triangleq \vec{x}_i$.
    \item Pass $\vec{z}_1, \dots, \vec{z}_{n}$ into a randomly initialized component $F$ that is trained on the task of interest that outputs the prediction $\hat{y}$. 
\end{enumerate}
\noindent \textbf{Pretrained Contextualized Representations.}
In this thesis, we use a pretrained ELMo \citep{elmo} encoder to specify $c$. ELMo is a pretrained shallowly-bidirectional language model that was pretrained on 30 million sentences, or roughly one billion words, using the 1B Word Benchmark \citep{chelba2013}. The input is first tokenized and then each word is deconstructed into its corresponding character sequences. Each character is embedded independently and then a convolutional neural network is used to encode the sequence and produce word representations. The resulting word representations are passed through a two-layer left-to-right LSTM and a two-layer right-to-left LSTM. Two representations are produced for every word $w_i$. The first is  $\texttt{ELMo}_1\left(w_i \mid \langle w_1 \dots w_n \rangle \right) \in \mathbb{R}^{d}$, which is the concatenated hidden states from the first layer of both LSTMs. Similarly, the second is $\texttt{ELMo}_2\left(w_i \mid \langle w_1 \dots w_n \rangle \right) \in \mathbb{R}^{d}$, which is the concatenated hidden states from the second layer of both LSTMs. In our work,  $\vec{x}_i \triangleq \left[\texttt{ELMo}_1\left(w_i \mid \langle w_1 \dots w_n \rangle \right);\texttt{ELMo}_2\left(w_i \mid \langle w_1 \dots w_n \rangle \right)\right] \in \mathbb{R}^{2d}$.\footnote{$;$ denotes concatenation} \\

\noindent \textbf{Task-specific Component.} We decompose the task-specific model $F$ into a bidirectional LSTM, a max pooling layer, and a linear classifier:
\begin{align}\label{eq:task-specific-model}
    \overleftarrow{h_{1:n}}, \overrightarrow{h_{1:n}} \gets & \texttt{BidiLSTM}(\vec{z}_1, \dots, \vec{z}_n) \\
    \vec{h} \gets & \left[\max\left(\overleftarrow{h_{1:n}} \right)\ ;\  \max\left(\overrightarrow{h_{1:n}} \right) \right] \\
    \hat{y} \gets & \texttt{Softmax}\left(\mathbf{W}\vec{h} + \vec{b} \right)
\end{align}
where $\texttt{BidiLSTM}, \mathbf{W}, \vec{b}$ are all learnable parameters and $\max(\cdot)$ is the element-wise max operation. 
\section{Data}\label{sec:data}
In this work, we evaluate our methods and baselines on five single-sentence text classification datasets. Summary statistics regarding the distributional properties of these datasets are reported in \autoref{tab:dataset-standard-statistics}. We use official dataset splits\footnote{For datasets where the standard split is only into two partitions, we further partition the training set into $\frac{8}{9}$ training data and $\frac{1}{9}$ validation data.} when available and otherwise split the dataset randomly into $80\%$ training data, $10\%$ validation data, and $10\%$ test data.\footnote{When we create data splits, we elect to ensure that the distribution over labels in each dataset partition is equal across partitions.} We also report the fraction of examples that the \texttt{spaCy} parser generates an invalid parse. For examples with invalid parse, since we cannot meaningfully compute optimal linear layouts, we back-off to using the identity linear layout $\pi_I$. All data is publicly available and means for accessing the data are described in \autoref{appendix:reproducibility-data-access}. \\

\begin{table*}[t]
\centering
\begin{tabular}{cccccccc}
\toprule
& Train & Validation & Test & $\frac{\text{Words}}{\text{ex.}}$ & Unique Words & Classes & Fail $\%$ \\
\midrule
\texttt{CR} & 3016 & 377 & 378 & 20 & 5098 & 2 & 19.2\\  
\texttt{SUBJ} & 8000 & 1000 & 1000 & 25 & 20088 & 2 & 13.6\\ 
\texttt{SST-2} & 6151 & 768 & 1821 & 19 & 13959 & 2 & 6.7\\ 
\texttt{SST-5} & 7594 & 949 & 2210 & 19 & 15476 & 5 & 6.8\\ 
\texttt{TREC} & 4846 & 605 & 500 & 10 & 9342 & 6  & 1.0\\
\bottomrule
\end{tabular}
\caption{Summary statistics for text classification datasets. Train, validation, and test refer to the number of examples in the corresponding dataset partition. $\frac{\text{Words}}{\text{ex.}}$ refers to the average number of words in the input sentence for each example in the union of the training data and the validation data. Unique words is the number of unique words in the union of the training data and the validation data. Classes is the size of the label space for the task. Fail $\%$ is the percentage of examples in the union of the training and validation set where the \texttt{spaCy} dependency parser emits an invalid parse (e.g~multiple syntactic roots, not a connected graph). }
\label{tab:dataset-standard-statistics}
\end{table*}

\noindent \textbf{Customer Reviews Sentiment Analysis.} This dataset was introduced by \citet{cr} as a collection of web-scraped customer reviews of products. The task is to predict whether a given review is positive or negative. This dataset will be abbreviated \texttt{CR} hereafter. \\

\noindent \textbf{Movie Reviews Subjectivity Analysis} This dataset was introduced by \citet{subj} as a collection of movie-related texts from \url{rottentomatoes.com} and \url{imdb.com}. The task is to predict whether a given sentence is subjective or objective. Subjective examples were sentences extracted from movie reviews from \url{rottentomatoes.com} and objectives examples were sentences extracted from movie plot summaries from \url{imdb.com}. This dataset will be abbreviated \texttt{SUBJ} hereafter. \\

\noindent \textbf{Movie Reviews Sentiment Analysis} This dataset was introduced by \citet{sst} as the \textit{Stanford Sentiment Treebank}. The dataset extends the prior dataset of \citet{mr}, which is a set of movie reviews from \url{rottentomatoes.com}, by re-annotating them using Amazon Mechanical Turk. In the binary setting, the task is to predict whether the review is positive or negative. In the fine-grained or five-way setting, the task is to predict whether the review is negative, somewhat negative, neutral, somewhat positive, or positive. This dataset will be abbreviated as \texttt{SST-2} to refer to the binary setting and as \texttt{SST-2} to refer to the five-way setting hereafter. \\

\noindent \textbf{Question Classification} This dataset was introduced by \citet{trec} as a collection of data for question classification and was an aggregation of data from \citet{hovy2001} and the TREC 8 \citep{trec8}, 9 \citep{trec9}, and 10 \citep{trec10} evaluations. The task is to predict the semantic class\footnote{The notion of a \textit{semantic class} for questions follows \citet{semanticquestion} and is distinct from the conceptual classes of \citet{conceptualquestion1,  conceptualquestion2, conceptualquestion3}. The distinction primarily centers on the handling of factual questions \citep[c.f.][]{trec}.}  of a question from the following categories: abbreviation, entity, description, human, location, and numerical value.\footnote{The data was annotated a course-grained granularity of six classes and a fine-grained granularity of 50 classes. In this work, we use the coarse-grained granularity.} This dataset will be abbreviated \texttt{TREC} hereafter. \\

\noindent \textbf{Why single-sentence text classification?}
Text classification is a standard and simple test-bed for validating that new methods in NLP have potential. It is also one of the least computationally demanding settings, which allowed us to be more comprehensive and rigorous in considering more datasets, ordering rules, and hyperparameter settings. In this work, we choose to further restrict ourselves to single-sentence datasets as the optimization we do is ill-posed. Consequently, without further constraints/structure, the solutions our algorithms generate may lack fluency across sentence boundaries. In the single-sentence setting, models must learn to grapple with this lack of systematicity in modelling across examples but the task is significantly simpler as there are no modelling challenges within a single example due to this lack of systematicity. As we were mainly interested in what could be done with relatively pure algorithmic optimization, we began with this simpler setting with fewer confounds. We discuss this as both a limitation and opportunity for future work in \autoref{chapter:conclusions}. 
\section{Experimental Conditions}\label{sec:experimental-conditions}
In this work, we tokenize input sentences using \texttt{spaCy} \citep{spacy}. We additionally dependency parse sentences using the \texttt{spaCy} dependency parser, which uses the CLEAR dependency formalism/annotation schema.\footnote{\url{https://github.com/clir/clearnlp-guidelines/blob/master/md/specifications/dependency_labels.md}} We then embed tokens using ELMo pretrained representations. We then freeze these representations, permute them according to the order we are studying, and then pass them through our bidirectional LSTM encoder to fine-tune for the specific task we are considering. We use a linear classifier with a \texttt{Softmax} function on top of the concatenated max-pooled hidden states produced by the LSTM. \\

\noindent \textbf{Why} \texttt{spaCy}\textbf{?}
Using \texttt{spaCy} for both dependency parsing and tokenizing ensured that there were no complications due to misalignment between the dependency parser's training data's tokenization and our tokenization. Further, the \texttt{spaCy} dependencer parser is reasonably high-quality for English and is easily used off-the-shelf. \\

\noindent \textbf{Why ELMo?} 
ELMo representations are high-quality pretrained representations. While superior pretrained representations exist (e.g.~BERT), we chose to use ELMo as it is was compatible with the tokenization schema of \texttt{spaCy} and introduced no further complications due to subwords. While in other works of ours, we rigorously mechanisms for converting from subword-level to word-level representations for BERT and other transformers \citep{interpretingpretrainedcwrs}, we chose to avoid this complication in this work as it was not of interest and was merely a confound. \\

\noindent \textbf{Why frozen representations?} 
\citet{to-tune-or-not} provides compelling and comprehensive evidence the frozen representations perform better than fine-tuned representations when using ELMo as the pretrained model. Further, in this work we are interested in integrating our novel word orders with pretrained representations and wanted to isolate this effect from possible further confounds due to the nature of fine-tuning with a different word order from the pretraining word order. \\

\noindent \textbf{Why bidirectional LSTM encoder?} \\
As we discuss in \autoref{subsec:sequentialmodels}, LSTMs have proven to be reliably and performant encoders across a variety of NLP tasks. Further, bidirectional LSTMs have almost uniformly led to further improvemenets, as we discuss in \autoref{subsec:alternative}. In our \texttt{pretrain-permute-finetune} framework, it was necessary to use a task-specific encoder that was not order-agnostic as otherwise the permutations would have no effect.

\noindent \textbf{Why max-pooling?} 
In initial experiments, we found that the pooling decision had an unclear impact on results but that there was marginal evidence to suggest that max-pooling outperformed averaging (which is the only other pooling choice we considered; we did not consider the concatenation of the two as in \citet{ulmfit}). Further, recent work that specifically studies ELMo for text classification with extensive hyperparameter study also uses max-pooling \citep{to-tune-or-not}. Separately, \citet{maxpool} demonstrate that max-pooling may have both practical advantages and theoretical justification. 

\subsection{Hyperparameters}
We use input representations from ELMo that are $2048$ dimensional. We use a single layer  bidirectional LSTM with output dimensionality $h$ and dropout \citep{dropout} is introduced to the classifier input with dropout probability $p$ as form of regularization. The weights of the LSTM and classifier are initialized according to random samples from PyTorch default distribution.\footnote{\url{https://pytorch.org/docs/stable/nn.init.html}} Optimization is done using the Adam optimizer \citep{adam} and the standard cross-entropy loss, which has proven to be a robust pairing of (optimizer, objective) in numerous NLP applications. Optimizer parameters are set using PyTorch defaults.\footnote{\url{https://pytorch.org/docs/stable/optim.html}} Examples are presented in minibatches of size $16$. The stopping condition for training is the model after training for $12$ epochs. We found this threshold after looking at performance across several different epochs (those in $\{3,6,9,12,15\}$). We found that standard early-stopping methods did not reliably lead to improvements, perhaps due to the fact that models converge so quickly (hence early-stopping and a fixed threshold are near-equal). All hyperparameters that were specified above were optimized on the \texttt{SUBJ} dataset using the identity word order ($r_I$) to ensure this baseline was as well-optimized as possible. Parameter choices for $h$ and $p$ were initially optimized over $\{16, 32, 64, 128, 256\} \times \{0.0, 0.02, 0.2\}$ for the \texttt{SUBJ} (one of the easiest datasets) and \texttt{SST-5} (one of the hardest datasets), again using $r_I$ to ensure the optimization was done to favor the baseline. From these 15 candidate hyperparameter settings, the six highest performing were chosen\footnote{We did this initial reduction of the hyperparameter space due to computational constraints. In doing so, we tried to ensure that the settings yielded the strongest possible baselines.} and all results are provided for these settings (for all word orders and all datasets) in \autoref{chapter:appendix-additional-results}. All results reported subsequently are for hyperparameters individually optimized for the (word order, dataset) pair being considered. All further experimental details are deferred to \autoref{appendix:reproducibility-experimental-details}. 

\section{Results and Analysis}\label{sec:results+analysis}
\noindent \textbf{Orders.} We analyze the results for the following eight orders:
\begin{enumerate}
    \item $r_I$ --- Each sentence is ordered as written.
    \item $r_r$ --- Each sentence is ordered according to a linear layout sampled from the uniform distribution over $S_n$.
    \item $r_b$ --- Each sentence is ordered according to the Cuthill-McKee heuristic to minimize bandwidth.
    \item $r_c$ --- Each sentence is ordered according to the algorithm of \citet{cutwidth-nlogn/linear} to minimize cutwidth. The linear layout is constrained to be optimal among all projective orderings.
    \item $r_m$ --- Each sentence is ordered according to the algorithm of \citet{minlaprojective} to minimize the minimum linear arrangement objective. The linear layout is constrained to be optimal among all projective orderings.
    \item $r_{\tilde{b}}$ --- Each sentence is ordered according the \texttt{Transposition Monte Carlo} algorithm to minimize bandwidth.
    \item $r_{\tilde{c}}$ --- Each sentence is ordered according the \texttt{Transposition Monte Carlo} algorithm to minimize cutwidth.
    \item $r_{\tilde{m}}$ --- Each sentence is ordered according the \texttt{Transposition Monte Carlo} algorithm to minimize the minimum linear arrangement objective.
\end{enumerate}
\noindent \textbf{Optimization effects.} Beyond the standard distributional properties of interest in NLP for datasets (\autoref{tab:dataset-standard-statistics}), it is particularly pertinent to consider the effects of our optimization algorithms on the bandwidth, minimum linear arrangement score, and cutwidth across these datasets. We report this in \autoref{tab:dataset-optimization-scores}. \\

\begin{table*}[t]
\centering
\fontsize{6.95}{8}\selectfont{
\begin{tabular}{c|ccc|ccc|ccc|ccc|ccc}
\toprule
& \multicolumn{3}{c|}{\texttt{CR}} & \multicolumn{3}{c|}{\texttt{SUBJ}} & \multicolumn{3}{c|}{\texttt{SST-2}} & \multicolumn{3}{c|}{\texttt{SST-5}} & \multicolumn{3}{c}{\texttt{TREC}}  \\
& B & C & M & B & C & M & B & C & M & B & C & M & B & C & M \\
\midrule
$r_r$         & 16.03 & 10.07 & 146.9 & 20.42 & 13.20 & 221.9 & 15.92 & 10.92 & 153.1 & 15.84 & 10.90 & 151.6 & 7.55  & 6.05  & 39.24  \\  
$r_I$         & 11.52 & 4.86  & 49.87 & 16.12 & 5.49  & 69.52 & 13.16 & 5.03  & 54.29 & 13.04 & 5.02  & 53.84 & 6.87  & 3.95  & 21.99  \\ 
\midrule
$r_b$         & 5.44  & 5.44  & 55.21 & 6.37  & 6.37  & 78.94 & 5.52  & 5.52  & 58.20 & 5.50  & 5.50  & 57.68 & 3.36  & 3.36  & 17.76  \\  
$r_c$         & 6.58  & 3.34  & 34.68 & 8.41  & 3.62  & 46.78 & 6.54  & 3.20  & 35.69 & 6.51  & 3.20  & 35.43 & 3.57  & 2.45  & 14.76  \\ 
$r_m$         & 6.19  & 3.35  & 34.13 & 7.69  & 3.64  & 45.70 & 6.00  & 3.21  & 34.90 & 5.98  & 3.21  & 34.65 & 3.46  & 2.45  & 14.62  \\  
\midrule
$\tilde{r}_b$ & 6.64  & 5.29  & 55.84 & 8.68  & 6.40  & 81.12 & 7.11  & 5.61  & 61.74 & 7.08  & 5.57  & 60.97 & 3.84  & 3.75  & 21.88  \\  
$\tilde{r}_c$ & 10.42 & 4.02  & 47.00 & 14.60 & 4.57  & 66.03 & 11.66 & 4.08  & 50.89 & 6.96  & 4.07  & 50.44 & 3.79  & 3.00  & 19.66  \\ 
$\tilde{r}_m$ & 6.85  & 3.29  & 35.68 & 8.60  & 3.66  & 49.17 & 7.00  & 3.29  & 37.64 & 11.57 & 3.29  & 37.32 & 5.77  & 2.54  & 15.40  \\ 
\bottomrule
\end{tabular}}
\caption{Bandwidth (B), cutwidth (C), and minimum linear arrangement (M) scores for every (dataset, ordering rule) pair considered.}
\label{tab:dataset-optimization-scores}
\end{table*}

We begin by considering the relationship between the random word orders $r_r$ and the standard English word orders $r_I$ (top band of \autoref{tab:dataset-optimization-scores}). In particular, we observe that across all five datasets, standard English substantially optimizes these three costs better than a randomly chosen ordering would. In the case of minimum linear arrangement, this provides further evidence to corpus analyses conducted by \citet{futrell2015}. Similarly, for the bandwidth and cutwidth objectives, this suggests that these objectives at least correlate with costs that humans may optimize for in sentence production and processing. \\

We then consider the optimization using existing algorithms in the literature as compared to standard English and random word orders (top and middle bands of \autoref{tab:dataset-optimization-scores}). We observe that across all datasets, all optimized ordering rules perform random word orders across all objectives. While it is unsurprising that the optimal order for a given objective outperforms the other orders and standard English, we do note that the margins are quite substantial in comparing standard English to each rule. That is to say, standard English can still be substantially further optimized for any of the given orders. Additionally, we consider the scores associated for an ordering rule that are \textit{not} being optimized for. In particular, we see that optimizing for either cutwidth or minimum linear arrangement yields similar scores across all three objectives and all five datasets. We hypothesize this is due to both algorithms have a shared algorithmic subroutine (the disjoint strategy). Optimizing for either order yields substantial improvements over standard English across all three objectives and only marginally underperforms the bandwidth-optimal order $r_b$. On the other hand, we find an interesting empirical property that the cutwidth and bandwidth scores for $r_b$ are consistently equal. This may suggest that this is a theoretical property of the algorithm that be formally proven. Further, $r_b$ generally (except for \texttt{TREC}) yields greater minimum linear arrangements compared to English. \\

\begin{table*}[t]
\centering
\fontsize{6.95}{8}\selectfont{
\begin{tabular}{c|ccc|ccc|ccc|ccc|ccc}
\toprule
& \multicolumn{3}{c|}{\texttt{CR}} & \multicolumn{3}{c|}{\texttt{SUBJ}} & \multicolumn{3}{c|}{\texttt{SST-2}} & \multicolumn{3}{c|}{\texttt{SST-5}} & \multicolumn{3}{c}{\texttt{TREC}}  \\
& B & C & M & B & C & M & B & C & M & B & C & M & B & C & M \\
\midrule
$r_r$         & 16.03 & 10.07 & 146.9 & 20.42 & 13.20 & 221.9 & 15.92 & 10.92 & 153.1 & 15.84 & 10.90 & 151.6 & 7.55  & 6.05  & 39.24  \\  
$r_I$         & 11.52 & 4.86  & 49.87 & 16.12 & 5.49  & 69.52 & 13.16 & 5.03  & 54.29 & 13.04 & 5.02  & 53.84 & 6.87  & 3.95  & 21.99  \\ 
\midrule
$r_b$         & 5.44  & 5.44  & 55.21 & 6.37  & 6.37  & 78.94 & 5.52  & 5.52  & 58.20 & 5.50  & 5.50  & 57.68 & 3.36  & 3.36  & 17.76  \\  
$r_c$         & 6.58  & 3.34  & 34.68 & 8.41  & 3.62  & 46.78 & 6.54  & 3.20  & 35.69 & 6.51  & 3.20  & 35.43 & 3.57  & 2.45  & 14.76  \\ 
$r_m$         & 6.19  & 3.35  & 34.13 & 7.69  & 3.64  & 45.70 & 6.00  & 3.21  & 34.90 & 5.98  & 3.21  & 34.65 & 3.46  & 2.45  & 14.62  \\  
\midrule
$\tilde{r}_b$ & 6.64  & 5.29  & 55.84 & 8.68  & 6.40  & 81.12 & 7.11  & 5.61  & 61.74 & 7.08  & 5.57  & 60.97 & 3.84  & 3.75  & 21.88  \\  
$\tilde{r}_c$ & 10.42 & 4.02  & 47.00 & 14.60 & 4.57  & 66.03 & 11.66 & 4.08  & 50.89 & 6.96  & 4.07  & 50.44 & 3.79  & 3.00  & 19.66  \\ 
$\tilde{r}_m$ & 6.85  & 3.29  & 35.68 & 8.60  & 3.66  & 49.17 & 7.00  & 3.29  & 37.64 & 11.57 & 3.29  & 37.32 & 5.77  & 2.54  & 15.40  \\ 
\bottomrule
\end{tabular}}
\caption{Duplicated from \autoref{tab:dataset-optimization-scores} for convenience. Bandwidth (B), cutwidth (C), and minimum linear arrangement (M) scores for every (dataset, ordering rule) pair considered.}
\label{tab:dataset-optimization-scores-copy}
\end{table*}

Next, we consider the optimization using the algorithms we introduce with the \texttt{Transposition Monte Carlo} method as compared to English and random word orders (top and bottom bands of \autoref{tab:dataset-optimization-scores}). We observe that the same comparison between random word orders and the word orders we introduce to optimize objectives holds as in the case of the algorithms from the prior literature. Further, the relationship between standard English and these heuristic-based word orders mimics the trends between standard English and the word orders derived from prior algorithms. However, we find that the margins are substantially smaller. This is not particularly surprising, as it suggests that our heuristics are less effective at pure optimization than the more well-established algorithms in the literature. \\

When we strictly consider the word orders generated by the heuristics we introduce (bottom band of \autoref{tab:dataset-optimization-scores}), we observe one striking finding immediately. In particular, the cutwidth objective evaluated on the cutwidth-minimizing order $\tilde{r}_c$ is often greater than the same objective evaluated on the minimum linear arrangement-minimizing order $\tilde{r}_m$. A similar result holds between $\tilde{r}_b$ and $\tilde{r}_m$ for the \texttt{SUBJ} and \texttt{SST-2} datasets. What this implies is that the greediness and transposition-based nature of \texttt{Transposition Monte Carlo} may more directly favor optimizing minimum linear arrangement (and that this coincides with minimizing the other objectives). Alternatively, this may suggest that a more principled and nuanced procedure is needed for the stopping parameter $T$ in the algorithm. \\

Finally, we discuss the relationship between the algorithms given in prior work (and the corresponding word orders) and the algorithms we present under the \texttt{Transposition Monte Carlo framework} (and the corresponding word orders). Here, we compare the middle and bottom bands of \autoref{tab:dataset-optimization-scores}. As observed previously, we see that for the score being optimized, the corresponding word order based on an algorithm from the literature outperforms the corresponding word order based on an algorithm we introduce. More broadly, we see the quality of the optimization across almost all objectives and datasets for all ordering rule pairs $r_k, \tilde{r}_k$ for $k \in \{b,c,m\}$ is better for $r_k$ than $\tilde{r}_k$. This is consistent with our intuition previously --- our heuristics sacrifice the extent to which they purely optimize the objective being considered to retain aspects of the original linear layout $\pi_I$. In the analysis of downstream results, we will address whether this compromise provides any benefits in modelling natural language computationally.  

\begin{table*}[t]
\centering
\large{
\begin{tabular}{cccccc}
\toprule
& \texttt{CR} & \texttt{SUBJ} & \texttt{SST-2} & \texttt{SST-5} & \texttt{TREC} \\
\midrule
$r_I$ & 0.852 & 0.955 & \textbf{0.896} & 0.485 & 0.962 \\
$r_r$ & 0.842 & 0.95 & 0.877 & 0.476 & 0.954 \\
\midrule
$r_b$ & {\color{magenta} \textit{0.854}} & 0.952 & 0.873 & 0.481 & {\color{magenta} \textit{0.966}} \\
$r_c$ & \textbf{0.86} & 0.953 & 0.874 & 0.481 & 0.958 \\
$r_m$ & 0.841 & 0.951 & 0.874 & 0.482 & {\color{magenta} \textit{0.962}} \\
\midrule
$\tilde{r}_b$ & {\color{magenta} \textit{0.852}} & 0.949 & 0.882 & 0.478 & 0.956 \\
$\tilde{r}_c$ & 0.849 & {\color{magenta} \textit{0.956}} & 0.875 & \textbf{0.494} & \textbf{0.968} \\
$\tilde{r}_m$ & 0.844 & \textbf{0.958} & 0.876 & 0.476 & {\color{magenta} \textit{0.962}} \\
\bottomrule
\end{tabular}}
\caption{Full classification results where the result reported is the max across hyperparameter settings. Results use \texttt{pretrain-permute-finetune} framework with the order specified in each row. All other hyperparameters are set as described previously. The top part of the table refers to baselines. The middle part of the table refers to orders derived from pure optimization algorithms. The bottom part of the table refers to orders derived from heuristic algorithms we introduce using \texttt{Transposition Monte Carlo}. The best performing ordering rule for a given dataset is indicated in \textbf{bold}. Any ordering rule (that is neither the best-performing order rule nor $r_I$) that performs at least as well as $r_I$ for a given dataset is indicated in {\color{magenta} \textit{italicized magenta}}. }
\label{tab:main-results}
\end{table*}

\noindent \textbf{Downstream Performance.} 
Given the framing and backdrop we have developed thus far, we next consider the downstream effects of our novel word orders. In \autoref{tab:main-results}, we report these results as well as the results for the random word order baseline $r_r$ and the standard English word order $r_I$. In particular, recall that the $r_I$ performance is reflective of the performance of the state-of-the-art paradigm in general: \texttt{pretrain-and-finetune}. \\

Each entry reflects the optimal choice of hyperparameters (among those we considered) for the corresponding model on the corresponding dataset (based on validation set performance). In \autoref{chapter:appendix-additional-results}, we provide additional results for all hyperparameter settings we studied\footnote{These results appear in Tables~\ref{tab:pool=twelve-p=0.0-h=32}--\ref{tab:pool=twelve-p=0.2-h=256}.} as well as results for all hyperparameters with the stopping condition of training for $15$ epochs (we observed no sustained improvements for any model using any order on any dataset beyond this threshold).\footnote{These results appear in Tables~\ref{tab:pool=fifteen-p=0.0-h=32}--\ref{tab:pool=fifteen-p=0.2-h=256}.}  \\

We begin by considering the comparison between the random word order $r_r$ and the standard English word order $r_I$ (top band of \autoref{tab:main-results}). Note that for $r_r$, we are using a bidirectional LSTM in the task-specific modelling as a type of set encoder. While previous works such as \citet{bommasani-sparse} have also used RNN variants in order-invariant settings, this is fairly nonstandard and requires that the model learns permutation equivariance given that the order bears no information. Unsurprisingly, we see that $r_I$ outperforms $r_r$ across all datasets. However, the margin is fairly small for all five datasets. From this, we can begin by noting that ELMo already is a powerful pretrained encoder and much of the task-specific modelling could have just been achieved by a shallow classifier on top of the ELMo representations. Further, we note that this attunes us to a margin that is fairly significant (the difference between whatever can be gained from the English word order and the pretrained contextualized word representations versus what can be gleaned only from the pretrained contextualized word representations). \\

\begin{table*}[t]
\centering
\large{
\begin{tabular}{cccccc}
\toprule
& \texttt{CR} & \texttt{SUBJ} & \texttt{SST-2} & \texttt{SST-5} & \texttt{TREC} \\
\midrule
$r_I$ & 0.852 & 0.955 & \textbf{0.896} & 0.485 & 0.962 \\
$r_r$ & 0.842 & 0.95 & 0.877 & 0.476 & 0.954 \\
\midrule
$r_b$ & {\color{magenta} \textit{0.854}} & 0.952 & 0.873 & 0.481 & {\color{magenta} \textit{0.966}} \\
$r_c$ & \textbf{0.86} & 0.953 & 0.874 & 0.481 & 0.958 \\
$r_m$ & 0.841 & 0.951 & 0.874 & 0.482 & {\color{magenta} \textit{0.962}} \\
\midrule
$\tilde{r}_b$ & {\color{magenta} \textit{0.852}} & 0.949 & 0.882 & 0.478 & 0.956 \\
$\tilde{r}_c$ & 0.849 & {\color{magenta} \textit{0.956}} & 0.875 & \textbf{0.494} & \textbf{0.968} \\
$\tilde{r}_m$ & 0.844 & \textbf{0.958} & 0.876 & 0.476 & {\color{magenta} \textit{0.962}} \\
\bottomrule
\end{tabular}}
\caption{Duplicated from \autoref{tab:main-results} for convenience. Full classification results where the result reported is the max across hyperparameter settings. Results use \texttt{pretrain-permute-finetune} framework with the order specified in each row. All other hyperparameters are set as described previously. The top part of the table refers to baselines. The middle part of the table refers to orders derived from pure optimization algorithms. The bottom part of the table refers to orders derived from heuristic algorithms we introduce using \texttt{Transposition Monte Carlo}. The best performing ordering rule for a given dataset is indicated in \textbf{bold}. Any ordering rule (that is neither the best-performing order rule nor $r_I$) that performs at least as well as $r_I$ for a given dataset is indicated in {\color{magenta} \textit{italicized magenta}}. }
\label{tab:main-results-copy}
\end{table*}

Next, we consider the comparison of the word orders derived via combinatorial optimization algorithms in the literature and the baselines of standard English and the randomized word order (top and middle bands of \autoref{tab:main-results}). We begin by noting that the optimized word orders \textit{do not} always outperform $r_r$. This is most salient in looking at the results for \texttt{SST-2}. While the margin is exceedingly small when the optimized word orders underperform against $r_r$, this suggests that the lack of systematicity or regularity in the input (perhaps due to the optimization being underdetermined) makes learning from the word order difficult. While none of $r_b, r_c, r_m$ consistently perform as well or better than $r_I$ ($r_b$ performs slightly better on two datasets, slightly worse on two datasets, and substantially worse on \texttt{SST-2} as, arguably, the best-performing of the three), there are multiple instances where, for specific datasets (\texttt{CR} and \texttt{TREC}), they do outperform $r_I$. This alone may suggest that a more careful design of ordering rules could yield improvements. Again, recall this is a highly controlled comparison with something that is a reasonable stand-in for the state-of-the-art regime of \texttt{pretrain-and-finetune}. \\

Looking at the comparison between the word orders derived via algorithms from the literature and their analogues derived from \texttt{Transposition Monte Carlo} variants (middle and bottom bands of \autoref{tab:main-results-copy}), we observe that neither set of orders strictly dominates the other. However, on four of the five datasets (all but \texttt{CR}), the best-performing order is one produced by \texttt{Transposition Monte Carlo}. This provides evidence to the claim that the sole goal in constructing order rules for downstream NLP of pure combinatorial optimization is insufficient and arguably naive. Instead, attempting a balance between retaining information encoded in the original linear layout $\pi_I$ while performing some reduction of dependency-related costs may be beneficial. We revisit this in \autoref{sec:futuredirections}. \\

Finally, we compare the word orders produced using our heuristic to the baselines (top and bottom bands of \autoref{tab:main-results-copy}). We begin by observing that in this case, almost every ordering rule outperforms the random word order (the sole exception being $\tilde{r}_b$ for \texttt{SUBJ} by a margin of just $0.001$). Further, while no single ordering rule reliably outperforms $r_I$ ($\tilde{r}_c$ significantly outperforms $r_I$ on two datasets, is significantly outperformed on one dataset, performs marginally better on one dataset, and performs marginally worse on dataset), we do find a word order that outperforms $r_I$ among $\tilde{r}_b, \tilde{r}_c, \tilde{r}_m$ in four of the five datasets (the exception being \texttt{SST-2}). In fact, the majority (i.e.~on three of the five datasets) of the best-performing ordering rules of all considered are word orders produced using \texttt{Transposition Monte Carlo}. This is additional evidence to suggest the merits of novel ordering rules (that are not standard English) for downstream NLP and that their design likely should try to maintain some of the information-encoding properties of standard English. 

\chapter{Conclusions}\label{chapter:conclusions}
To conclude, in this chapter, we review the findings presented. We further provide the open problems we find most pressing, the future directions we find most interesting, and the inherent limitations we find most important. 
\section{Summary} \label{sec:summary}
In this thesis, we have presented a discussion of the current understanding of word order as it pertains to human language. Dually, we provide a description of the set of approaches that have been taken towards modelling word order in computational models of language. \\

\noindent We then focus on how psycholinguistic approaches to word order can be generalized and formalized under a rich algorithmic framework. In particular, we concentrate on word orders that can be understood in terms of linear layout optimization using dependency parses as a scaffold that specifies the underlying graph structure. We describe existing algorithms for the \textsc{bandwidth}, \textsc{minLA}, and \textsc{cutwidth} problems with a special interest on algorithms when projectivity constraints are imposed. In doing so, we show that the algorithms of \citet{minlaprojective} and \citet{cutwidth-nlogn/linear} can be interpreted on a shared framework. Further, we correct some nuances in the algorithm and analysis of \citet{minlaprojective}. As a taste of how algorithmic interpretation could be used to reinterpret existing psycholinguistic approaches, we also prove \autoref{thm:minLA-two-interpretations}. This result establishes a connection between capacity and memory constraints that was previously unknown in the psycholinguistics literature (to the author's knowledge). To accompany these provable results, we also introduce  simple heuristic algorithms (i.e.~\texttt{Transposition Monte Carlo} algorithms for each of the three objectives we considered). These algorithms provide a mechanism for balancing pure algorithmic optimization with the additional constraints that might be of interest in language (e.g~preserving other types of information encoded in word order). \\

\noindent With these disparate foundations established, the last portion of this thesis discusses the relationship between novel/alternative word orders and computational models of language. We introduce the \texttt{pretrain-permute-finetune} framework to seamlessly integrate our novel orders into the \textit{de facto} \texttt{pretrain-and-finetune} paradigm. In particular, our approach incurs a near-trivial one-time cost and makes no further changes to the pretraining process, training and inference costs (both with respect to time and memory), or model architecture. We then examine the empirical merits of our method on several tasks (and across several hyperparameter settings), showing that our heuristic algorithms may sometimes outperform their provable brethren and that, more broadly, novel word orders can outperform using standard English. In doing so, we establish the grounds for considering the utility of this approach for other languages, for other tasks, and for other models.

\section{Open Problems} \label{sec:openproblems}
\noindent \textbf{Weighted optimization.} 
In this work, we consider orders specified based on objectives evaluated over the dependency parse $\mathcal{G} = (\mathcal{V}, \mathcal{E})$. In representing the parse in this way, we ignore the labels and directionality of the edges.\footnote{Recall how we defined $\mathcal{E}$ as the unlabelled and undirected version of the true edge set of the dependency parse, $\mathcal{E}_\ell$.} While using a bidirectional encoder in the finetuning step of the \texttt{pretrain-permute-finetune} framework may alleviate concerns with directionality, we are ultimately neglecting information made available in the dependency parse. We take this step in our approach since there are known algorithms for the unweighted and undirected version of the combinatorial optimization problems we study. Nonetheless, studying provable and heuristic algorithms for the weighted version of these problems may be of theoretical interest and of empirical value. In particular, this would allow one to specify a weighting on edges dependent on their dependency relation types, which may better represent how certain dependencies are more or less important for downstream modelling. \\

\noindent \textbf{Alternative scaffolds.} In this work, we use dependency parses as the syntactic formalism for guiding the optimization. While this choice was well-motivated given the longstanding literature on dependency locality effects and that the linear tree structure of dependency parses facilitates algorithmic framing in the language of linear layouts, alternative syntactic structures such as constituency trees or semantic formalisms such as AMR semantic parses are also interesting to consider. This yields natural open questions of what appropriate scaffolds are, what resulting objectives and algorithmic approaches are appropriate, and how both the linguistic and algorithmic primitive interact with the downstream NLP task. Further inquiry might also consider the merits of combining multiple approaches given the benefits of joint and multi-task learning and fundamental questions regarding what one linguistic representation may provide that is unavailable from others.  \\

\noindent \textbf{Information in Order.} A central scientific question in studying human language is understanding what information is encoded in word order and how this information is organized. In this work, given that we argue for the merits of representing sentences using two distinct orders (one during the embedding stage and one during the fine-tuning stage), this question is of added value. In particular, a formal theory of what information is available with one order and another (from an information theoretic perspective) might be of value. Arguable an even more interesting question revolves around considering the ease of extracting the information given one ordering or another. Such a framing of orders facilitating the ease of extraction of information (from computational agents with bounded capacities, which is quite different from standard information theory) might thereafter induce a natural optimization-based approach to studying order --- select the order that maximizes the eases of extraction of certain information that is of interest.

\section{Future Directions} \label{sec:futuredirections}
In the previous section, we state open problems that are of interest and merit study. In this section, we provide more concrete future directions and initial perspectives on how to operationalize these directions. \\

\noindent \textbf{End-to-end permutation learning.} In this work, we generate novel word orders that are task-agnostic and that optimizes objectives that are not directly related with downstream performance. While there are fundamental questions about the relationships between discrete algorithmic ideas, classical combinatorial optimization, and modern deep learning, our approach runs stylistically counter to the dominant paradigms of NLP at present. In particular, the notion of attention emerged in machine translation as a soft way of aligning source and target and has proven to empirically more effective than harder alignments. Dually, such end-to-end optimization of attention weights implies that attention can be optimized in a task-specific way, which likely implies better performance. For these reason, studying end-to-end methods for (differentiably) finding orders may be particularly attractive as this would imply the permutation component could be dropped into existing end-to-end differentiable architecture. While reinforcement learning approaches to permutation generation/re-ordering may be natural, we believe a promising future direction would be direct learning of permutations. In this sense, a model would learn permutations of the input that correspond with improved downstream performance. 

At first glance, such optimization of permutations seems hard to imagine given that permutations are sparse/discrete and therefore unnatural for differentiable optimization. A recent line of work in the computer vision and machine learning communities has proposed methods towards differentiable/gradient-based optimization \citep{visual-permutation, latent-permutations}. In particular, many of these methods consider a generalization of the family of permutation matrices to the family of double stochastic matrices (DSMs). Importantly, DSMs specify a polytope (the Birkhoff polytope) which can be reached in a differentiable way via Sinkhorn iterations \citep{sinkhorn1, sinkhorn2}. Given these works, a natural question is whether these methods can be extended to sequential data (which are not of a fixed length, unlike image data), perhaps with padding as is done in Transformers, and what that might suggest for downstream NLP performance. \\

\noindent \textbf{Additional constraints to induce well-posed optimization.} In our approach, the solution to the optimization is not necessarily unique. In particular, for all three objectives, the reverse of an optimal sequence is also optimal. Further, \citet{minlaprojective} claim (without proof) that there can be $2^{\frac{n}{2}}$ optimal solutions for the optimization problem they study. Given that there may be many optima and potentially exponentially many, downstream models are faced with uncertainty about the structure of the input for any given input. In particular, the word orders we are considering may lack regularity and systematicity properties that are found in natural languages, though the structure of the disjoint strategy may help alleviate this.\footnote{They also may lack overt linear proxies for compositionality to the extent found in the corresponding natural language.} Developing additional constraints to further regularize the optima may benefit modelling as different examples will have more standardized formats. Given the observations we discuss in \autoref{sec:languageuniversals} regarding harmonic word orders improving language acquisition for child language learners (by appealing to their inductive biases), such as those of \citet{culbertson2017}, enforcing word order harmonies globally may be a natural means for increases regularity and systematicity in the resultant word orders. Similarly, such approaches may provide additional benefits when extended to the (more general) setting where inputs (may) contain multiple sentences.  \\

\noindent \textbf{Information locality.} In  \autoref{subsec:joint-theories}, we discuss the recent proposal for information locality as a generalization of dependency locality. As syntactic dependencies are clearly only a subset of the linguistic information of interest for a downstream model, a richer theory of locality effects and a broad classes of information may prove fruitful both in study human language processing and in motivating further novel orders. More broadly, we specifically point to the work of \citep{information-theory-computational-constraints} which introduces the $\mathcal{V}$-information theory. In particular, their work builds the standard preliminaries and primitives of Shannon's information theory \citep{informationtheory} with the notion of computational constraints. While they envision this in the context of motivating modern representation learning, we believe this may also be better theoretical machinery for information theoretic treatments of human language processing. After all, humans also have severely bounded processing capabilities. \\

\noindent \textbf{Understanding impact of novel orders.} In this work, we put forth a rigorous and highly controlled empirical evaluation of the word orders we introduce. However, our results and analysis do not explain \textit{why} and \textit{how} the novel word orders lead to improved performance. In particular, a hypothesis that serves as an undercurrent for this work is that the novel word orders we introduce \textit{simplify} computational processing by reducing the need for modelling long-distance dependencies and/or many dependencies simultaneously. Future work that executes a more thorough and nuanced analysis of model behavior and error patterns would help to validate the extent to which this hypothesis is correct. Additionally, it would help to disambiguate what modelling improvements that the work yields coincide with other modelling techniques and what improvements are entirely orthogonal to suggest how this work can integrate with other aspects of the vast literature on computational modelling in NLP.  \\

\noindent \textbf{Broader NLP evaluation.} Perhaps the most obvious future direction is to consider a broader evaluation of our framework. In particular, evaluating across pretraining architectures, fine-tuning architectures, datasets, tasks, and languages are all worthwhile for establishing the boundaries of this method's viability and utility. Additionally, such evaluations may help to further refine the design of novel word orders and suggest more nuanced procedures for integrating them (especially in settings where the dominant paradigm itself of \texttt{pretrain-and-finetune} is insufficient). Additionally, we posit that there may be natural interplay with our work and work on cross-lingual and multi-lingual methods that may especially merit concentrated study. Concretely, one could imagine applying the same ordering rule cross-linguistically and therefore normalizing some of the differences across languages at the input level. Consequently, there may be potential for this to improve alignment (e.g.~between word embeddings, between pretrained encoders) for different language or better enable cross-lingual transfer/multi-lingual representation learning.  
\section{Consequences}
Previously, we have summarized our contributions (\autoref{sec:summary}). In this section, we instead provide more abstract lessons or suggestive takeaways from this work.\\

\noindent \textbf{Word order in NLP.} In this work, we clearly substantiate that there is tremendous scope for improving how word order is considered within NLP and explicitly illuminate well-motivated directions for further empirical work. While our empirical results legitimately and resolutely confirm that theoretically-grounded algorithmic optimization may coincide with empirical NLP improvements, it remains open to what extent this is pervasive across NLP tasks and domains. \\

\noindent \textbf{Generalizing measures of language (sub)optimality.} In this work, we put forth alternative objectives beyond those generally considered in the dependency length minimization literature that may merit further psycholinguistic inquiry. In particular, by adopting the generality of an algorithmic framework, one can naturally pose many objectives that intuitively may align with human optimization in language production and comprehension. Further, in this work we clarify the extent to which human language is optimal and we note that the suboptimality of human language may be equally useful for understanding language's properties as its optimality. \\

\noindent \textbf{Interlacing psycholinguistic or algorithmic methods with NLP.} In this work, we present work that is uniquely at the triple intersection of psycholinguistics, algorithms, and NLP. While such overtly interdisciplinary work is of interest, we also argue that there is great value in considering the interplay between psycholinguistics and NLP or between algorithms and NLP, if not all three simultaneously. There is a long-standing tradition of connecting algorithms/computational theory with computational linguistics and NLP in the study of grammars and parsing \citep{chomsky1957syntactic-structures, chomskysyntax, kay-1967-experiments, earley-parser, joshi75, charniak1983parser, pereira-warren-1983-parsing, kay1986parsing, steedman87, kay-1989-head, eisner-1996-three, collins-1996-new, collins-1997-three, charniak-etal-1998-edge, gildea2002, collins-2003-head, klein-manning-2003-accurate, klein-manning-2003-parsing, steedman2004, mcdonald-etal-2005-non, mcdonald-pereira-2006-online, mitkov2012, chomsky2014minimalist, chomsky2014aspects}. Trailblazers of the field, such as Noam Chomsky and the recently-passed Arvind Joshi made great contributions in this line. Similarly, there are has been a recent uptick in the borrowing of methods from psycholinguistics to rigorously study human language processing\footnote{The first blackbox language learner we have tried to understand.} to interpret and understand neural models of language \citep{linzen-etal-2016-assessing, gulordava-etal-2018-colorless, van-schijndel-linzen-2018-neural, wilcox-etal-2018-rnn, marvin-linzen-2018-targeted, ravfogel-etal-2019-studying, van-schijndel-etal-2019-quantity, mccoy-etal-2019-right, futrell-levy-2019-rnns, wilcox-etal-2019-hierarchical, futrell-etal-2019-neural, van-schijndel-linzen-2019-entropy, prasad-etal-2019-using, ettinger2020, davis2020}.\footnote{The second blackbox language learner we have tried to understand.} And once again, another seminal mind of the field, Ron Kaplan, spoke to precisely this point in his keynote talk "\textit{Computational Psycholinguistics}" \citep{kaplan2019} at ACL 2019 as he received the most recent \href{https://www.aclweb.org/portal/content/announcement-2019-acl-lifetime-achievement-award-lta}{ACL Lifetime Achievement Award}. However, in the present time, beyond parsing and interpretability research, we see less interplay between either algorithms or psycholinguistics and NLP. Perhaps this thesis may serve as an implicit hortative to reconsider this.

\section{Limitations} \label{sec:limitations}
In the spirit of diligent science, we enumerate the limitations we are aware of with this work. We attempt to state these limitations fairly without trying to undercut or lessen their severity. We also note that we have considered the ethical ramifications of this work\footnote{This was inspired by the NeurIPS 2020 authorship guidelines, which require authors to consider the social and ethical impacts of their work.} and remark that we did not come to find any ethical concerns.\footnote{While not an ethical concern of the research, we do note that all figures and tables in this work are designed to be fully colorblind-friendly. In particular, while many figures/tables use color, they can also be interpreted using non-color markers and these markers are referenced in the corresponding caption.} \\

\noindent \textbf{Dependence on Dependencies.} In this work, we generate novel word orders contingent on access to a dependency parse. Therefore, the generality of our method is dependent on access to such a parse. Further, since we exclusively consider English dependency parsers, which are substantially higher quality than dependency parsers for most languages and especially low-resource languages, and datasets which are drawn from similar data distributions as the training data of the dependency parse, it is unclear how well are model with perform in setting where weaker dependency parsers are available or there are domain-adaptation concerns. \\

\noindent \textbf{Rigid Optimization.} By design, our approach is to purely optimize an objective pertaining to dependency locality. While the heuristic algorithms we introduce under the \texttt{Transposition Monte Carlo} framework afford some consideration of negotiating optimization quality with retention of the original sentence's order (through the parameter $T$), our methods provide no natural recipe for integration of arbitrary metadata or auxiliary information/domain knowledge. \\

\noindent \textbf{Evaluation Settings.} As we note in \autoref{sec:experimental-conditions}, we elect to evaluate on single-sentence text classification datasets. While we do provide valid and legitimate reasons for this decision, it does imply that the effectiveness of our methods for other tasks and task types (e.g.~sequence labelling, natural language generation, question answering) is left unaddressed. Further, this means that the difficulties of intra-example irregularity across sentences (due to optimization not being sufficiently determined/constrained) are not considered. As a more general consideration, \citet{linzen2020} argues that the evaluation protocol used in this work may not be sufficient for reaching the desired scientific conclusions regarding word order. \\

\noindent \textbf{Alternative Orders in Pretraining.} 
Due to computational constraints, we only study introducing permutations/re-orderings we generate after embedding using a pretrained encoder. However, a more natural (and computationally intensive/environmentally detrimental \citep{environment}) approach that may offer additional value is to use the orders during pretraining. As such, our work does not make clear that \texttt{pretrain-permute-finetune} is the best way to introduce our permutations into the standard \texttt{pretrain-and-finetune} regime and fails to consider an "obvious" alternative. \\

\noindent \textbf{English Language Processing.} In studying the effects of our orders empirically, we only use English language data.\footnote{In this work, we did not reconsider the source of the pretraining data for ELMo or the datasets we used in evaluating our method. Therefore we are not certain, but it seems likely that data mainly represents Standard American English and not other variants of English such as African American Vernacular English.} Several works have shown that performance on English data is neither inherently representative nor likely to be representative of performance in language processing across the many languages of the world \citep{bender2009, bender2011, bender2012, joshi2020}. Further, recent works have shown that English may actually be favorable for approaches that make use of LSTMs and RNNs \citep{dyer2019, davis2020}, which are models we explicitly use in this work. And finally, the entire premise of dependency locality heavily hinges on typological features of the language being studied.\footnote{This can be inferred from our discussion of dependency locality and word ordering effects in various natural languages in \autoref{chapter:wordorderinhlp}; a particularly salient example is that the notion of dependency locality is quite different in morphologically-rich languages as information is often conveyed through morphological markers and intra-word units rather than word order.} Therefore, the results of this work are severely limited in terms of claims that can be made for natural languages that are not English.   
\bibliography{main}
\appendix
\chapter{Reproducibility}\label{appendix:reproducibility}
In this appendix, we provide further details required to exactly reproduce this work. Particularly significant is that we release the code (\autoref{appendix:reproducibility-code-release}) for running all experiments and generating all tables/figures/visualizations used in this work. We adhere to the guidelines presented in \citet{dodge2019}, which were further extended in the EMNLP 2020 reproducibility guidelines\footnote{\url{https://2020.emnlp.org/call-for-papers}}, to provide strong and rigorous guarantees on the reproducibility of our work.
\section{Additional Experimental Details}\label{appendix:reproducibility-experimental-details}
We use Python 3.6.9 throughout this work along with PyTorch 1.5.0. \\

\noindent \textbf{Tokenization and Dependency Parsing.} Tokenization is done using the English \texttt{en\_core\_web\_lg} model released in \texttt{spaCy} version 2.2.4.  Dependency parsing is done using the same English \texttt{en\_core\_web\_lg} model released in \texttt{spaCy} version 2.2.4. The model is 789MB and is trained on OntoNotes 5 using a multi-task\footnote{The other tasks are part-of-speech tagging and named entity recognition.} convolutional neural network-based model with pretrained word embeddings initialized using GloVe.\footnote{The GloVe embeddings are trained on Common Crawl data.} \\

\noindent \textbf{Data Preprocessing.} Beyond tokenizing the data, we do no further pre-processing except removing ill-formed examples in any datasets (where there is an input and no label or vice versa). We find that there are $4$ such examples in the \texttt{CR} dataset and none in any of the other four datasets. \\

\noindent\textbf{Pretrained Representations.} We use pretrained ELMo representations that are are obtained by using data tokenized using \texttt{spaCy} as described previously. The exact pretrained ELMo encoders are available here\footnote{\url{https://s3-us-west-2.amazonaws.com/allennlp/models/elmo/2x4096_512_2048cnn_2xhighway/elmo_2x4096_512_2048cnn_2xhighway_weights.hdf5}; file name describes model parameters under the AllenNLP naming conventions.} and concatenate the representations from each of the two layers (yielding $2048$-dimensional vectors). \\

\noindent\textbf{Randomness.} We fix the Python and PyTorch random seeds to be random seed $0$. \\ 

\noindent\textbf{Reverse Cuthill-McKee.} We use the implementation of the algorithm provided in \texttt{SciPy} 1.4.1. \\ 
\section{Code Release}\label{appendix:reproducibility-code-release}
All code for this work is made publicly available. The code is hosted at \url{https://github.com/rishibommasani/MastersThesis}. We note that we provide documentation for most core functionality and clarifiications can be provided upon request.

\section{Data Access}\label{appendix:reproducibility-data-access}
All data used in this work is publicly available. The copies of the datasets we use are available at \url{https://github.com/harvardnlp/sent-conv-torch/tree/master/data} via the Harvard NLP group. The data can also be accessed from the corresponding websites for each of the datasets:
\begin{itemize}
    \item \texttt{CR} --- \url{https://www.cs.uic.edu/~liub/FBS/sentiment-analysis.html#datasets} \\
    Hosted by Bing Liu.
    \item \texttt{SUBJ} --- \url{https://www.cs.cornell.edu/people/pabo/movie-review-data/} \\ 
    Hosted by Lillian Lee.
    \item \texttt{SST-2} --- \url{https://nlp.stanford.edu/sentiment/} \\
    Hosted by Stanford NLP group.
    \item \texttt{SST-5} --- \url{https://nlp.stanford.edu/sentiment/} \\
    Hosted by Stanford NLP group.
    \item \texttt{TREC} --- \url{https://cogcomp.seas.upenn.edu/Data/QA/QC/} \\
    Hosted by Dan Roth and UPenn CogComp group.
\end{itemize}

\section{Contact Information}\label{appendix:reproducibility-questions-concerns}
Questions, concerns, and errata should be directed to the thesis author at any of:
\begin{itemize}
    \item \href{mailto:nlprishi@stanford.edu}{nlprishi@stanford.edu}
    \item \href{mailto:rb724@cornell.edu}{rb724@cornell.edu} 
    \item \href{mailto:rishibommasani@gmail.com}{rishibommasani@gmail.com}
\end{itemize}
Any and all remaining errors in this thesis are strictly due to the author. 
\chapter{Additional Results}\label{chapter:appendix-additional-results}
In this appendix, we provide further results that were not included in the main body of the thesis. We first provide the results for all hyperparameters with the stopping condition we used in the main body of thesis: stopping after the fixed threshold of 12 epochs. These results appear in Tables~\ref{tab:pool=twelve-p=0.0-h=32}--\ref{tab:pool=twelve-p=0.2-h=256}. For further completeness, we provide the results for all hyperparameters with the stopping condition after which we never saw any improvements (for all models, datasets, and orders): stopping after the fixed threshold of 15 epochs. These results appear in Tables~\ref{tab:pool=fifteen-p=0.0-h=32}--\ref{tab:pool=fifteen-p=0.2-h=256}. 
\begin{table*} \
\centering
\large{
\begin{tabular}{cccccc}
\toprule
& \texttt{CR} & \texttt{SUBJ} & \texttt{SST-2} & \texttt{SST-5} & \texttt{TREC} \\
\midrule
$r_I$ & 0.833 & 0.95 & 0.87 & 0.464 & 0.95 \\
$r_r$ & 0.823 & 0.947 & 0.87 & 0.442 & 0.944 \\
\midrule
$r_b$ & 0.839 & 0.941 & 0.867 & 0.452 & 0.962 \\
$r_c$ & 0.852 & 0.95 & 0.873 & 0.453 & 0.952 \\
$r_m$ & 0.836 & 0.946 & 0.862 & 0.456 & 0.948 \\
\midrule
$\tilde{r}_b$ & 0.839 & 0.945 & 0.879 & 0.478 & 0.954 \\
$\tilde{r}_c$ & 0.833 & 0.952 & 0.875 & 0.464 & 0.952 \\
$\tilde{r}_m$ & 0.841 & 0.958 & 0.876 & 0.448 & 0.954 \\
\bottomrule
\end{tabular}}
\caption{Full classification results for $h = 32, p = 0.0$. Results use \texttt{pretrain-permute-finetune} framework with the order specified in each row. All other hyperparameters are set as described previously. The top part of the table refers to baselines. The middle part of the table refers to orders derived from pure optimization algorithms. The bottom part of the table refers to orders derived from heuristic algorithms we introduce using \texttt{Transposition Monte Carlo}. }
\label{tab:pool=twelve-p=0.0-h=32}
\end{table*}

\begin{table*} \
\centering
\large{
\begin{tabular}{cccccc}
\toprule
& \texttt{CR} & \texttt{SUBJ} & \texttt{SST-2} & \texttt{SST-5} & \texttt{TREC} \\
\midrule
$r_I$ & 0.833 & 0.955 & 0.882 & 0.481 & 0.956 \\
$r_r$ & 0.840 & 0.945 & 0.864 & 0.458 & 0.954 \\
\midrule
$r_b$ & 0.854 & 0.948 & 0.865 & 0.481 & 0.958 \\
$r_c$ & 0.831 & 0.942 & 0.871 & 0.478 & 0.95 \\
$r_m$ & 0.831 & 0.95 & 0.864 & 0.464 & 0.958 \\
\midrule
$\tilde{r}_b$ & 0.839 & 0.946 & 0.866 & 0.46 & 0.952 \\
$\tilde{r}_c$ & 0.849 & 0.956 & 0.873 & 0.46 & 0.958 \\
$\tilde{r}_m$ & 0.831 & 0.949 & 0.868 & 0.457 & 0.962 \\
\bottomrule
\end{tabular}}
\caption{Full classification results for $h = 64, p = 0.02$. Results use \texttt{pretrain-permute-finetune} framework with the order specified in each row. All other hyperparameters are set as described previously. The top part of the table refers to baselines. The middle part of the table refers to orders derived from pure optimization algorithms. The bottom part of the table refers to orders derived from heuristic algorithms we introduce using \texttt{Transposition Monte Carlo}. }
\label{tab:pool=twelve-p=0.02-h=64}
\end{table*}

\begin{table*} \
\centering
\large{
\begin{tabular}{cccccc}
\toprule
& \texttt{CR} & \texttt{SUBJ} & \texttt{SST-2} & \texttt{SST-5} & \texttt{TREC} \\
\midrule
$r_I$ & 0.783 & 0.95 & 0.868 & 0.477 & 0.956 \\
$r_r$ & 0.823 & 0.94 & 0.876 & 0.446 & 0.952 \\
\midrule
$r_b$ & 0.786 & 0.937 & 0.869 & 0.471 & 0.962 \\
$r_c$ & 0.852 & 0.946 & 0.861 & 0.474 & 0.948 \\
$r_m$ & 0.836 & 0.918 & 0.868 & 0.456 & 0.95 \\
\midrule
$\tilde{r}_b$ & 0.841 & 0.947 & 0.864 & 0.478 & 0.956 \\
$\tilde{r}_c$ & 0.844 & 0.946 & 0.87 & 0.47 & 0.952 \\
$\tilde{r}_m$ & 0.825 & 0.948 & 0.872 & 0.461 & 0.96 \\
\bottomrule
\end{tabular}}
\caption{Full classification results for $h = 64, p = 0.2$. Results use \texttt{pretrain-permute-finetune} framework with the order specified in each row. All other hyperparameters are set as described previously. The top part of the table refers to baselines. The middle part of the table refers to orders derived from pure optimization algorithms. The bottom part of the table refers to orders derived from heuristic algorithms we introduce using \texttt{Transposition Monte Carlo}. }
\label{tab:pool=twelve-p=0.2-h=64}
\end{table*}

\begin{table*} \
\centering
\large{
\begin{tabular}{cccccc}
\toprule
& \texttt{CR} & \texttt{SUBJ} & \texttt{SST-2} & \texttt{SST-5} & \texttt{TREC} \\
\midrule
$r_I$ & 0.852 & 0.951 & 0.884 & 0.485 & 0.946 \\
$r_r$ & 0.840 & 0.95 & 0.877 & 0.476 & 0.952 \\
\midrule
$r_b$ & 0.841 & 0.95 & 0.873 & 0.481 & 0.956 \\
$r_c$ & 0.86 & 0.953 & 0.874 & 0.481 & 0.954 \\
$r_m$ & 0.836 & 0.951 & 0.874 & 0.482 & 0.962 \\
\midrule
$\tilde{r}_b$ & 0.847 & 0.947 & 0.882 & 0.469 & 0.95 \\
$\tilde{r}_c$ & 0.847 & 0.953 & 0.871 & 0.494 & 0.962 \\
$\tilde{r}_m$ & 0.828 & 0.951 & 0.876 & 0.467 & 0.962 \\
\bottomrule
\end{tabular}}
\caption{Full classification results for $h = 128, p = 0.02$. Results use \texttt{pretrain-permute-finetune} framework with the order specified in each row. All other hyperparameters are set as described previously. The top part of the table refers to baselines. The middle part of the table refers to orders derived from pure optimization algorithms. The bottom part of the table refers to orders derived from heuristic algorithms we introduce using \texttt{Transposition Monte Carlo}. }
\label{tab:pool=twelve-p=0.02-h=128}
\end{table*}

\begin{table*} \
\centering
\large{
\begin{tabular}{cccccc}
\toprule
& \texttt{CR} & \texttt{SUBJ} & \texttt{SST-2} & \texttt{SST-5} & \texttt{TREC} \\
\midrule
$r_I$ & 0.847 & 0.945 & 0.874 & 0.455 & 0.952 \\
$r_r$ & 0.839 & 0.942 & 0.851 & 0.419 & 0.952 \\
\midrule
$r_b$ & 0.825 & 0.948 & 0.865 & 0.445 & 0.95 \\
$r_c$ & 0.852 & 0.944 & 0.867 & 0.469 & 0.954 \\
$r_m$ & 0.817 & 0.942 & 0.87 & 0.441 & 0.954 \\
\midrule
$\tilde{r}_b$ & 0.831 & 0.943 & 0.874 & 0.468 & 0.948 \\
$\tilde{r}_c$ & 0.849 & 0.948 & 0.873 & 0.462 & 0.952 \\
$\tilde{r}_m$ & 0.844 & 0.933 & 0.863 & 0.473 & 0.958 \\
\bottomrule
\end{tabular}}
\caption{Full classification results for $h = 128, p = 0.2$. Results use \texttt{pretrain-permute-finetune} framework with the order specified in each row. All other hyperparameters are set as described previously. The top part of the table refers to baselines. The middle part of the table refers to orders derived from pure optimization algorithms. The bottom part of the table refers to orders derived from heuristic algorithms we introduce using \texttt{Transposition Monte Carlo}. }
\label{tab:pool=twelve-p=0.2-h=128}
\end{table*}

\begin{table*} \
\centering
\large{
\begin{tabular}{cccccc}
\toprule
& \texttt{CR} & \texttt{SUBJ} & \texttt{SST-2} & \texttt{SST-5} & \texttt{TREC} \\
\midrule
$r_I$ & 0.836 & 0.951 & 0.896 & 0.476 & 0.962 \\
$r_r$ & 0.842 & 0.934 & 0.864 & 0.451 & 0.952 \\
\midrule
$r_b$ & 0.825 & 0.952 & 0.87 & 0.462 & 0.966 \\
$r_c$ & 0.836 & 0.946 & 0.864 & 0.462 & 0.958 \\
$r_m$ & 0.841 & 0.951 & 0.859 & 0.461 & 0.96 \\
\midrule
$\tilde{r}_b$ & 0.852 & 0.949 & 0.874 & 0.461 & 0.956 \\
$\tilde{r}_c$ & 0.825 & 0.932 & 0.874 & 0.467 & 0.968 \\
$\tilde{r}_m$ & 0.841 & 0.941 & 0.86 & 0.476 & 0.958 \\
\bottomrule
\end{tabular}}
\caption{Full classification results for $h = 256, p = 0.2$. Results use \texttt{pretrain-permute-finetune} framework with the order specified in each row. All other hyperparameters are set as described previously. The top part of the table refers to baselines. The middle part of the table refers to orders derived from pure optimization algorithms. The bottom part of the table refers to orders derived from heuristic algorithms we introduce using \texttt{Transposition Monte Carlo}. }
\label{tab:pool=twelve-p=0.2-h=256}
\end{table*}

\pagebreak 
\begin{table*}[t]
\centering
\large{
\begin{tabular}{cccccc}
\toprule
& \texttt{CR} & \texttt{SUBJ} & \texttt{SST-2} & \texttt{SST-5} & \texttt{TREC} \\
\midrule
$r_I$ & 0.841 & 0.948 & 0.867 & 0.462 & 0.944 \\
$r_r$ & 0.825 & 0.945 & 0.87 & 0.443 & 0.942 \\
\midrule
$r_b$ & 0.841 & 0.94 & 0.864 & 0.45 & 0.932 \\
$r_c$ & 0.844 & 0.948 & 0.87 & 0.462 & 0.948 \\
$r_m$ & 0.836 & 0.946 & 0.858 & 0.457 & 0.954 \\
\midrule
$\tilde{r}_b$ & 0.833 & 0.941 & 0.877 & 0.469 & 0.952 \\
$\tilde{r}_c$ & 0.828 & 0.951 & 0.873 & 0.466 & 0.95 \\
$\tilde{r}_m$ & 0.844 & 0.953 & 0.877 & 0.456 & 0.948 \\
\bottomrule
\end{tabular}}
\caption{Full classification results for $h = 32, p = 0.0$. Results are reported for models after they were trained for 15 epochs. Results use \texttt{pretrain-permute-finetune} framework with the order specified in each row. All other hyperparameters are set as described previously. The top part of the table refers to baselines. The middle part of the table refers to orders derived from pure optimization algorithms. The bottom part of the table refers to orders derived from heuristic algorithms we introduce using \texttt{Transposition Monte Carlo}. }
\label{tab:pool=fifteen-p=0.0-h=32}
\end{table*}

\begin{table*}[t]
\centering
\large{
\begin{tabular}{cccccc}
\toprule
& \texttt{CR} & \texttt{SUBJ} & \texttt{SST-2} & \texttt{SST-5} & \texttt{TREC} \\
\midrule
$r_I$ & 0.833 & 0.954 & 0.882 & 0.43 & 0.956 \\
$r_r$ & 0.841 & 0.945 & 0.863 & 0.459 & 0.95 \\
\midrule
$r_b$ & 0.852 & 0.95 & 0.871 & 0.439 & 0.956 \\
$r_c$ & 0.836 & 0.944 & 0.871 & 0.468 & 0.932 \\
$r_m$ & 0.831 & 0.95 & 0.862 & 0.463 & 0.958 \\
\midrule
$\tilde{r}_b$ & 0.847 & 0.948 & 0.864 & 0.459 & 0.95 \\
$\tilde{r}_c$ & 0.847 & 0.955 & 0.873 & 0.462 & 0.958 \\
$\tilde{r}_m$ & 0.833 & 0.949 & 0.865 & 0.461 & 0.96 \\
\bottomrule
\end{tabular}}
\caption{Full classification results for $h = 64, p = 0.02$. Results are reported for models after they were trained for 15 epochs. Results use \texttt{pretrain-permute-finetune} framework with the order specified in each row. All other hyperparameters are set as described previously. The top part of the table refers to baselines. The middle part of the table refers to orders derived from pure optimization algorithms. The bottom part of the table refers to orders derived from heuristic algorithms we introduce using \texttt{Transposition Monte Carlo}. }
\label{tab:pool=fifteen-p=0.02-h=64}
\end{table*}

\begin{table*}[t]
\centering
\large{
\begin{tabular}{cccccc}
\toprule
& \texttt{CR} & \texttt{SUBJ} & \texttt{SST-2} & \texttt{SST-5} & \texttt{TREC} \\
\midrule
$r_I$ & 0.823 & 0.951 & 0.874 & 0.458 & 0.956 \\
$r_r$ & 0.82 & 0.945 & 0.861 & 0.455 & 0.954 \\
\midrule
$r_b$ & 0.844 & 0.947 & 0.87 & 0.455 & 0.948 \\
$r_c$ & 0.847 & 0.947 & 0.871 & 0.465 & 0.964 \\
$r_m$ & 0.839 & 0.938 & 0.867 & 0.461 & 0.96 \\
\midrule
$\tilde{r}_b$ & 0.849 & 0.946 & 0.855 & 0.47 & 0.958 \\
$\tilde{r}_c$ & 0.836 & 0.957 & 0.855 & 0.47 & 0.948 \\
$\tilde{r}_m$ & 0.825 & 0.952 & 0.874 & 0.465 & 0.958 \\
\bottomrule
\end{tabular}}
\caption{Full classification results for $h = 64, p = 0.2$. Results are reported for models after they were trained for 15 epochs. Results use \texttt{pretrain-permute-finetune} framework with the order specified in each row. All other hyperparameters are set as described previously. The top part of the table refers to baselines. The middle part of the table refers to orders derived from pure optimization algorithms. The bottom part of the table refers to orders derived from heuristic algorithms we introduce using \texttt{Transposition Monte Carlo}. }
\label{tab:pool=fifteen-p=0.2-h=64}
\end{table*}

\begin{table*}[t]
\centering
\large{
\begin{tabular}{cccccc}
\toprule
& \texttt{CR} & \texttt{SUBJ} & \texttt{SST-2} & \texttt{SST-5} & \texttt{TREC} \\
\midrule
$r_I$ & 0.852 & 0.949 & 0.883 & 0.49 & 0.966 \\
$r_r$ & 0.833 & 0.949 & 0.876 & 0.466 & 0.946 \\
\midrule
$r_b$ & 0.844 & 0.948 & 0.873 & 0.477 & 0.944 \\
$r_c$ & 0.854 & 0.953 & 0.874 & 0.474 & 0.958 \\
$r_m$ & 0.839 & 0.95 & 0.87 & 0.452 & 0.916 \\
\midrule
$\tilde{r}_b$ & 0.847 & 0.947 & 0.878 & 0.422 & 0.95 \\
$\tilde{r}_c$ & 0.849 & 0.951 & 0.873 & 0.471 & 0.96 \\
$\tilde{r}_m$ & 0.831 & 0.952 & 0.875 & 0.451 & 0.958 \\
\bottomrule
\end{tabular}}
\caption{Full classification results for $h = 128, p = 0.02$. Results are reported for models after they were trained for 15 epochs. Results use \texttt{pretrain-permute-finetune} framework with the order specified in each row. All other hyperparameters are set as described previously. The top part of the table refers to baselines. The middle part of the table refers to orders derived from pure optimization algorithms. The bottom part of the table refers to orders derived from heuristic algorithms we introduce using \texttt{Transposition Monte Carlo}. }
\label{tab:pool=fifteen-p=0.02-h=128}
\end{table*}

\begin{table*}[t]
\centering
\large{
\begin{tabular}{cccccc}
\toprule
& \texttt{CR} & \texttt{SUBJ} & \texttt{SST-2} & \texttt{SST-5} & \texttt{TREC} \\
\midrule
$r_I$ & 0.844 & 0.944 & 0.868 & 0.467 & 0.96 \\
$r_r$ & 0.852 & 0.941 & 0.869 & 0.435 & 0.948 \\
\midrule
$r_b$ & 0.828 & 0.945 & 0.864 & 0.448 & 0.954 \\
$r_c$ & 0.857 & 0.917 & 0.87 & 0.462 & 0.96 \\
$r_m$ & 0.817 & 0.938 & 0.865 & 0.461 & 0.954 \\
\midrule
$\tilde{r}_b$ & 0.82 & 0.942 & 0.871 & 0.466 & 0.944 \\
$\tilde{r}_c$ & 0.849 & 0.953 & 0.845 & 0.482 & 0.956 \\
$\tilde{r}_m$ & 0.847 & 0.951 & 0.868 & 0.45 & 0.958 \\
\bottomrule
\end{tabular}}
\caption{Full classification results for $h = 128, p = 0.2$. Results are reported for models after they were trained for 15 epochs. Results use \texttt{pretrain-permute-finetune} framework with the order specified in each row. All other hyperparameters are set as described previously. The top part of the table refers to baselines. The middle part of the table refers to orders derived from pure optimization algorithms. The bottom part of the table refers to orders derived from heuristic algorithms we introduce using \texttt{Transposition Monte Carlo}. }
\label{tab:pool=fifteen-p=0.2-h=128}
\end{table*}

\begin{table*}[t]
\centering
\large{
\begin{tabular}{cccccc}
\toprule
& \texttt{CR} & \texttt{SUBJ} & \texttt{SST-2} & \texttt{SST-5} & \texttt{TREC} \\
\midrule
$r_I$ & 0.841 & 0.953 & 0.891 & 0.471 & 0.968 \\
$r_r$ & 0.844 & 0.945 & 0.873 & 0.462 & 0.948 \\
\midrule
$r_b$ & 0.825 & 0.955 & 0.841 & 0.455 & 0.96 \\
$r_c$ & 0.839 & 0.947 & 0.871 & 0.475 & 0.952 \\
$r_m$ & 0.839 & 0.953 & 0.868 & 0.483 & 0.956 \\
\midrule
$\tilde{r}_b$ & 0.849 & 0.95 & 0.861 & 0.466 & 0.956 \\
$\tilde{r}_c$ & 0.823 & 0.947 & 0.87 & 0.474 & 0.968 \\
$\tilde{r}_m$ & 0.839 & 0.947 & 0.87 & 0.471 & 0.956 \\
\bottomrule
\end{tabular}}
\caption{Full classification results for $h = 256, p = 0.2$. Results are reported for models after they were trained for 15 epochs. Results use \texttt{pretrain-permute-finetune} framework with the order specified in each row. All other hyperparameters are set as described previously. The top part of the table refers to baselines. The middle part of the table refers to orders derived from pure optimization algorithms. The bottom part of the table refers to orders derived from heuristic algorithms we introduce using \texttt{Transposition Monte Carlo}. }
\label{tab:pool=fifteen-p=0.2-h=256}
\end{table*}

\end{document}